\documentclass{article} %
\usepackage[final]{colm2026_conference}

\usepackage{microtype}
\usepackage{hyperref}
\usepackage{url}
\usepackage{booktabs}
\usepackage{multirow}
\usepackage{tabularx}
\usepackage{graphicx}
\usepackage{subcaption}
\usepackage{amsmath,amsfonts}
\usepackage{inconsolata}
\usepackage[capitalize,noabbrev]{cleveref}

\usepackage{enumitem}
\setlist[itemize]{
    leftmargin=*,       %
    labelsep=0.5em,     %
    nosep,              %
    itemindent=0pt,     %
    leftmargin=2.0em    %
}

\usepackage{amsmath, amssymb, amsthm}
\newtheorem{theorem}{Theorem}[section]
\newtheorem{proposition}[theorem]{Proposition}
\newtheorem{corollary}[theorem]{Corollary}

\newcommand{\nothink}{\texttt{/no\_think}}
\newcommand{\think}{\texttt{/think}}
\newcommand{\ple}{\textsc{PLE}}

\definecolor{darkgreen}{rgb}{0.0, 0.5, 0.0}
\definecolor{lowblue}{rgb}{0.1,0.2,0.5}
\definecolor{POneGreen}{RGB}{0,100,0} %
\definecolor{ArchiveOrange}{RGB}{204,102,0} %

\usepackage{lineno}

\definecolor{darkblue}{rgb}{0, 0, 0.5}
\hypersetup{colorlinks=true, citecolor=darkblue, linkcolor=darkblue, urlcolor=darkblue}

\title{Path-Lock Expert: Separating Reasoning Mode in Hybrid Thinking via Architecture-Level Separation}

\author{Shouren Wang$^{1,*}$, Wang Yang$^{1,*}$, Chuang Ma$^{2,3}$, Debargha Ganguly$^{1}$, Vikash Singh$^{1}$, \\ \textbf{
Chaoda Song$^{1}$, Xinpeng Li$^{1}$, Xianxuan Long$^{4}$, Vipin Chaudhary$^{1,\dagger}$, Xiaotian Han$^{1,\dagger}$ }\\
$^{1}$Case Western Reserve University \quad $^{2}$Kyoto University \quad $^{3}$NII LLMC \\ $^{4}$Michigan State University \\
\texttt{\{sxw992,wxy320,dxg512,vxs465,cxs965,xxl1337,vipin,xhan\}@case.edu} \\
\texttt{ma.chuang.52h@st.kyoto-u.ac.jp} \quad \texttt{longxia2@msu.edu} \\
$^{*}$Equal contribution \quad $^{\dagger}$Corresponding authors
}

\begin{document}

\ifcolmsubmission
\linenumbers
\fi

\maketitle

\begin{abstract}
Hybrid-thinking language models expose explicit \think{} and \nothink{} modes, but current designs do not separate them cleanly. Even in \nothink{} mode, models often emit long and self-reflective responses, causing \emph{reasoning leakage}. 
Existing work reduces this issue through better data curation and multi-stage training, yet leakage remains because both modes are still encoded in the same feed-forward parameters.
We propose \textbf{Path-Lock Expert (\ple{})}, an architecture-level solution that replaces the single MLP in each decoder layer with two semantically locked experts, one for \think{} and one for \nothink{}, while keeping attention, embeddings, normalization, and the language-model head shared.
A deterministic control-token router selects exactly one expert path for the entire sequence, so inference preserves the dense model's per-token computation pattern and each expert receives mode-pure updates during supervised fine-tuning.
Across math and science reasoning benchmarks, \ple{} maintains strong \think{} performance while producing a substantially stronger mode separation---a \nothink{} mode with higher accuracy and far less reasoning leakage.
On Qwen3-4B, for example, compared to the SFT-only baseline on AIME24, \ple{} generates $17\times$ fewer reflective tokens ($6.01$ vs.\ $0.35$ per answer) and $2\times$ shorter outputs ($8665$ vs.\ $4101$ tokens), and improves \nothink{} accuracy from $35.33\%$ to $44.67\%$, while maintaining \think{}-mode performance ($61.33\%$ vs.\ $60.00\%$).
These results suggest that controllable hybrid thinking is fundamentally an architectural problem, and separating mode-specific feed-forward pathways is a simple and effective solution. 
The code is available at: \url{https://github.com/SR-A-W/path-lock-expert}
\end{abstract}

\begin{figure}[ht]
    \centering
    \includegraphics[width=1.0\linewidth]{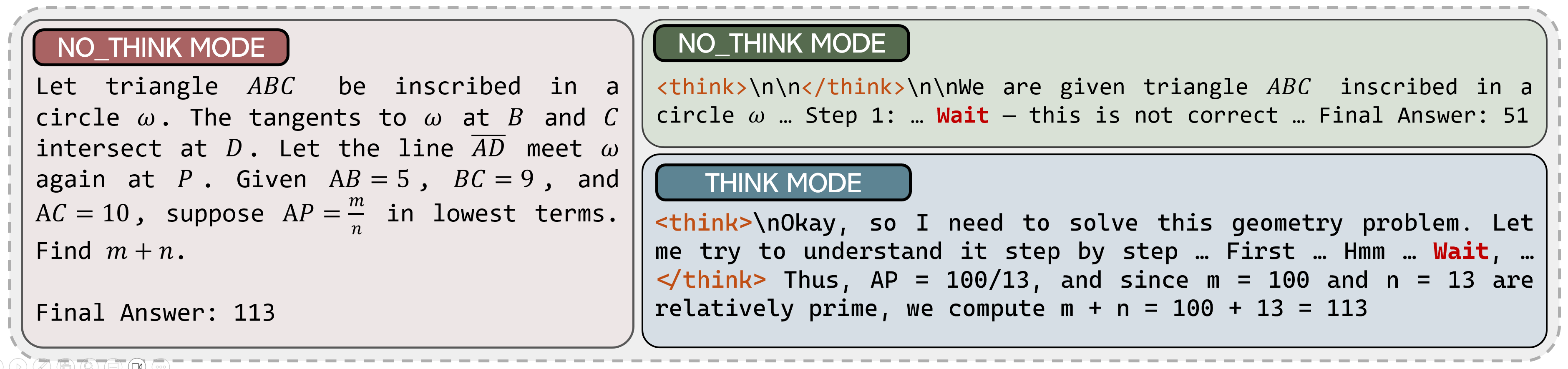}
    \caption{Motivating example of reasoning leakage. On an AIME24 problem, Qwen3-8B in \nothink{} mode still emits self-reflective tokens (e.g., \textcolor{red}{\texttt{Wait}}) outside the empty \texttt{<think>} block and arrives at an incorrect answer, while \think{} mode solves it correctly. This shows that current hybrid-thinking models cannot fully separate the different reasoning modes.}
    \label{fig:placeholder}
\end{figure}

\section{Introduction}\label{sec:introduction}
\begin{table}[t]
\centering
\caption{Reasoning leakage in hybrid-thinking models. We compare Qwen2.5-7B-Instruct (Instruct, a pure instruction-tuned model) with Qwen3-4B (Hybrid, a hybrid-thinking model). We report accuracy (Acc.), average output length (Len.), and average reflective tokens per response (\#Refl./Ans.). $\Delta$ denotes the difference between each model's \nothink{} performance and the Instruct baseline. The hybrid model's \nothink{} mode stays markedly longer and more reflective than a genuine direct-answer mode.}
\label{tab:mov_all}
\vspace{-8pt}
\resizebox{\textwidth}{!}{
\setlength{\tabcolsep}{3pt}
\begin{tabular}{l|ccc|ccc|ccc}
\toprule
\multirow{2}{*}{Model}
& \multicolumn{3}{c|}{Acc. (\%)} & \multicolumn{3}{c|}{Len.} & \multicolumn{3}{c}{\#Refl./Ans.} \\
\cmidrule(lr){2-4} \cmidrule(lr){5-7} \cmidrule(lr){8-10}
& Think & No-think & $\Delta$ & Think & No-think & $\Delta$ & Think & No-think & $\Delta$ \\
\midrule
\multicolumn{10}{c}{MATH500} \\
\midrule
Instruct & -- & \textcolor{lowblue}{59.94} & & -- & \textcolor{lowblue}{703} & & -- & \textcolor{lowblue}{0} &  \\
Hybrid & 92.82 & \textcolor{lowblue}{82.60} & \textcolor{lowblue}{\textbf{+22.66}} & 4384 & \textcolor{lowblue}{1583} & \textcolor{lowblue}{\textbf{+880}} & 16.74 & \textcolor{lowblue}{0.04} & \textcolor{lowblue}{\textbf{+0.04}} \\
\midrule
\multicolumn{10}{c}{AIME24} \\
\midrule
Instruct & -- & \textcolor{lowblue}{6.67} & & -- & \textcolor{lowblue}{1729} & & -- & \textcolor{lowblue}{0} &  \\
Hybrid & 68.00 & \textcolor{lowblue}{23.33} & \textcolor{lowblue}{\textbf{+16.66}} & 16466 & \textcolor{lowblue}{12506} & \textcolor{lowblue}{\textbf{+10777}} & 98.35 & \textcolor{lowblue}{0.17} & \textcolor{lowblue}{\textbf{+0.17}} \\
\midrule
\multicolumn{10}{c}{GPQA} \\
\midrule
Instruct & -- & \textcolor{lowblue}{30.15} & & -- & \textcolor{lowblue}{775} & & -- & \textcolor{lowblue}{0} &  \\
Hybrid & 51.01 & \textcolor{lowblue}{46.97} & \textcolor{lowblue}{\textbf{+16.82}} & 11371 & \textcolor{lowblue}{2679} & \textcolor{lowblue}{\textbf{+1904}} & 188.59 & \textcolor{lowblue}{0.05} & \textcolor{lowblue}{\textbf{+0.05}} \\
\bottomrule
\end{tabular}
}
\end{table}
Hybrid-thinking language models are only useful if their control tokens actually control behavior. In principle, \think{} should invoke deliberate reasoning when a problem requires it, while \nothink{} should return a direct answer when extended reasoning is unnecessary~\citep{yang2025qwen3,gemini25,agarwal2025gpt}. This distinction matters because long reasoning traces come with an associated cost: they increase latency, consume token budget, and occupy context-window space~\citep{sui2025stop,chen2024not,zhang2025lightthinker}. A practical \nothink{} mode therefore must not be a weakened version of \think{}; it must operate as a genuine direct-answer mode.

In current hybrid-thinking systems, this separation is weak. Even under explicit \nothink{} instructions, models often continue to emit signs of deliberation: responses become unnecessarily long, reflective markers such as \texttt{wait} or \texttt{hmm} appear, and outputs drift toward partial chain-of-thought generation. Prior literature defined this failure mode as \textbf{reasoning leakage}. Under this failure mode, \nothink{} becomes an unstable halfway mode between direct answering and explicit reasoning, thereby the interface no longer delivers the efficiency, predictability, and behavioral control that motivate hybrid thinking in the first place.

Most existing attempts to fix leakage operate through training alone. Prior work shows that changing data scale, \think{} vs. \nothink{} ratios, and multi-stage training schedules can reduce leakage substantially~\citep{wang2025demystifying}. However, lower leakage alone does not guarantee a strong \nothink{} mode: answers may remain verbose, partially reflective, or weaker in accuracy. This suggests that the limitation is not only in data or optimization. A single dense decoder is being asked to realize two competing output behaviors with the same feed-forward parameters; one mode should externalize reasoning, while the other should suppress it. We therefore hypothesize that weak mode separation is partly an architectural interference problem.

We address this with \textbf{Path-Lock Expert (\ple{})}, a minimal architecture that makes the mode choice part of the architecture itself. \ple{} replaces each decoder MLP with two semantically locked experts, one dedicated to \think{} and one dedicated to \nothink{}, and uses deterministic control-token routing to activate exactly one expert path for the entire sequence. Attention, embeddings, normalization, and the language-model head remain shared. This design separates the parameters most directly tied to generation behavior while preserving a shared representational backbone, and it does so without a learned router, auxiliary balancing losses, or a complex RL-based training pipeline.

Across mathematical and scientific reasoning benchmarks, \ple{} yields a substantially stronger \nothink{} mode than dense baselines and prior training-only hybrid approaches: no-think outputs are shorter, exhibit far less reasoning leakage, and achieve better no-think accuracy in our primary evaluations. At the same time, \ple{} maintains strong \think{}-mode performance. Our claim is intentionally specific: we do not argue that \ple{} solves hybrid thinking in general or establishes a new overall state of the art; rather, we show that architecture-level mode separation can materially improve the usefulness of the \nothink{} mode while preserving the value of \think{}. Our contributions are as follows:
\begin{itemize}
    \item We motivate an architectural view of reasoning leakage in hybrid-thinking models, arguing that training heuristics alone may be insufficient because shared feed-forward parameters interfere with clean mode separation.
    \item We introduce \ple{}, a deterministic dual-expert decoder architecture that separates \think{} and \nothink{} behavior at the MLP level while preserving a shared Transformer backbone and dense-model-like per-token inference computation.
    \item We show that, across our evaluated settings, \ple{} strengthens the \nothink{} mode through lower leakage, shorter outputs, and improved no-think accuracy in the primary evaluations, while maintaining strong \think{} performance.
\end{itemize}

\section{Motivation}\label{sec:motivation}
We define \textbf{hybrid thinking} as an interface in which a single model exposes two response behaviors---an extended chain-of-thought reasoning mode (\think{}) and a concise direct-answer mode (\nothink{}), termed the \textit{thinking} and \textit{non-thinking} modes in the Qwen3 Technical Report~\citep{yang2025qwen3}---selected by a user-supplied control token appended to the query, following that report and \citet{wang2025demystifying}. This explicit, user-directed setting is distinct from \emph{implicit-routing} variants, where the model decides internally. Our problem statement follows the explicit setting: \emph{can a hybrid-thinking model truly switch between \think{} and \nothink{} modes---in particular, given that the user has selected \nothink{}, can the model reliably yield a short and direct answer?}

We use a simple empirical comparison: we compare a pure instruction-tuned model, an off-the-shelf hybrid-thinking model, and the strongest prior training-level mitigation, and ask whether \nothink{} behaves like a genuine direct-answer mode. First, \nothink{} should suppress explicit reasoning artifacts such as long deliberative outputs and reflective phrases. Second, it should preserve the behavioral character of direct answering rather than merely producing a weaker version of \think{}.

\paragraph{Off-the-shelf hybrid thinking leaks by default.}
\Cref{tab:mov_all} shows that released hybrid-thinking models do learn to respond differently under \think{} and \nothink{} prompts, but the separation is incomplete.
Relative to the pure Instruct baseline, the hybrid model's \nothink{} outputs remain longer and contain nonzero reflective-token counts across benchmarks.
This means the issue is not that hybrid models fail to react to the control token at all; rather, they react only partially.
\nothink{} behaves like a softened form of thinking, not like a truly distinct direct-answer mode.

\paragraph{Training-only fixes help, but they stop short of true separation.}
\citet{wang2025demystifying} systematically explored the most natural mitigation strategies: scaling training data, changing the think/no-think ratio, and applying multi-stage training.
Their best recipe significantly reduces leakage, confirming that data and optimization matter.
However, it still does not match the behavioral target defined above.
Reflective tokens remain nonzero, and the resulting \nothink{} behavior does not consistently recover the direct-answer profile of the pure Instruct model across benchmarks.
This is the key empirical observation: even after substantial recipe engineering, the model remains only partially controllable.

\section{Method: Path-Lock Expert}\label{sec:method}

\begin{figure}[t]
    \centering
    \includegraphics[width=1.0\linewidth]{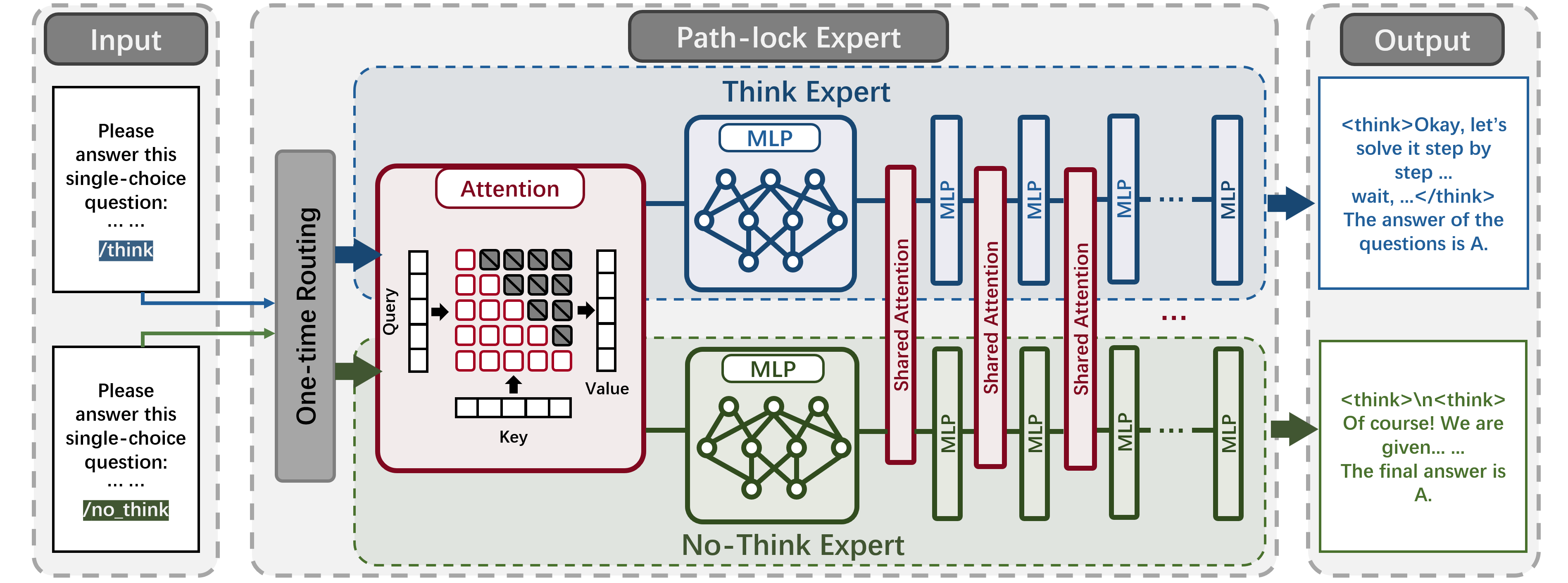}
    \caption{Path-Lock Expert replaces the single MLP in each decoder layer with two mode-specific experts while keeping the attention backbone shared. A single routing decision is made from the control tokens and reused across all layers and all decoding steps.}
    \label{fig:architecture}
\end{figure}

Path-Lock Expert (\ple{}) aims to achieve three primary objectives in hybrid-thinking language models: 1) better reasoning mode separation than conventional hybrid-thinking architectures; 2) performance preservation, ensuring the model's core reasoning capabilities are not degraded; and 3) parameter efficiency, avoiding the prohibitive cost of training or deploying two independent models. \ple{} achieves these with a minimal modification to a standard decoder-only Transformer, described below.

\textbf{Path-locked dual-expert decoder.}
In every layer, the single MLP is replaced by two structurally identical copies---one for \think{} and one for \nothink{}. All other components (self-attention, positional encoding, normalization, embeddings, LM head) remain shared. Because exactly one expert is active per response, per-token computation is unchanged. By replicating only the MLP rather than duplicating the full model, \ple{} saves approximately 33\% of total parameters relative to deploying two independent copies of the backbone.

\textbf{Deterministic control-token routing.}
The active expert is selected by the last \texttt{/think} or \texttt{/no\_think} token appearing in the input; the choice is resolved once and locked for all layers and all decoding steps. No learned router or balancing loss is involved. This one-time, all-layer lock is what motivates the name Path-Lock Expert: swapping the control token flips a single routing index and activates the other expert, leaving attention, embeddings, norms, and the LM head unchanged. We formalize the detection procedure, its string-level fallback, and edge-case handling, and also reported detailed validation tests in Appendix~\ref{sec:appendix_routing_pseudocode}.

\textbf{Routing-conditioned fine-tuning.}
PLE is fine-tuned with standard causal LM loss on chat examples tagged with either mode. The inactive expert receives zero gradient; the shared backbone is updated by both modes. This removes direct parameter coupling between the two MLP pathways while preserving a common representational substrate. We provide a rigorous definition in Appendix~\ref{sec:formal_architecture_specification}.

\paragraph{Relation to Mixture-of-Experts.}
While previous work such as \textit{Metis-HOME} \citep{lan2025metishome} has explored using a two-expert MoE for mode separation, their approach still relies on the \textbf{conventional MoE framework} with learned, token-level routers. In contrast, \ple{} shifts from token-level competition to \textbf{deterministic, sequence-level routing}. By tying expert identity directly to the \think{} or \nothink{} interface via control tokens, \ple{} eliminates the need for router parameters, balancing losses, or complex dispatch decisions. This ensures a ``locked''  path that is structurally stable and specifically designed for behavioral control.

\section{Experiments}\label{sec:experiments}

In this section, we evaluate the effectiveness of \ple{} in achieving architectural mode separation. Our experiments are designed to verify whether physical parameter isolation can improve mode separation by addressing three key questions: 
1) \textbf{Leakage Reduction:} Does \ple{} significantly reduce reasoning leakage (reflective tokens) in \nothink{} mode compared to dense baselines? 
2) \textbf{Mode Performance:} Does \ple{} improve the performance of the \nothink{} mode by reducing gradient interference? 
3) \textbf{Reasoning Preservation:} Does the architectural separation maintain the original \think{} performance without degradation?

\begin{figure}[t]
    \centering
    \includegraphics[width=1.0\linewidth]{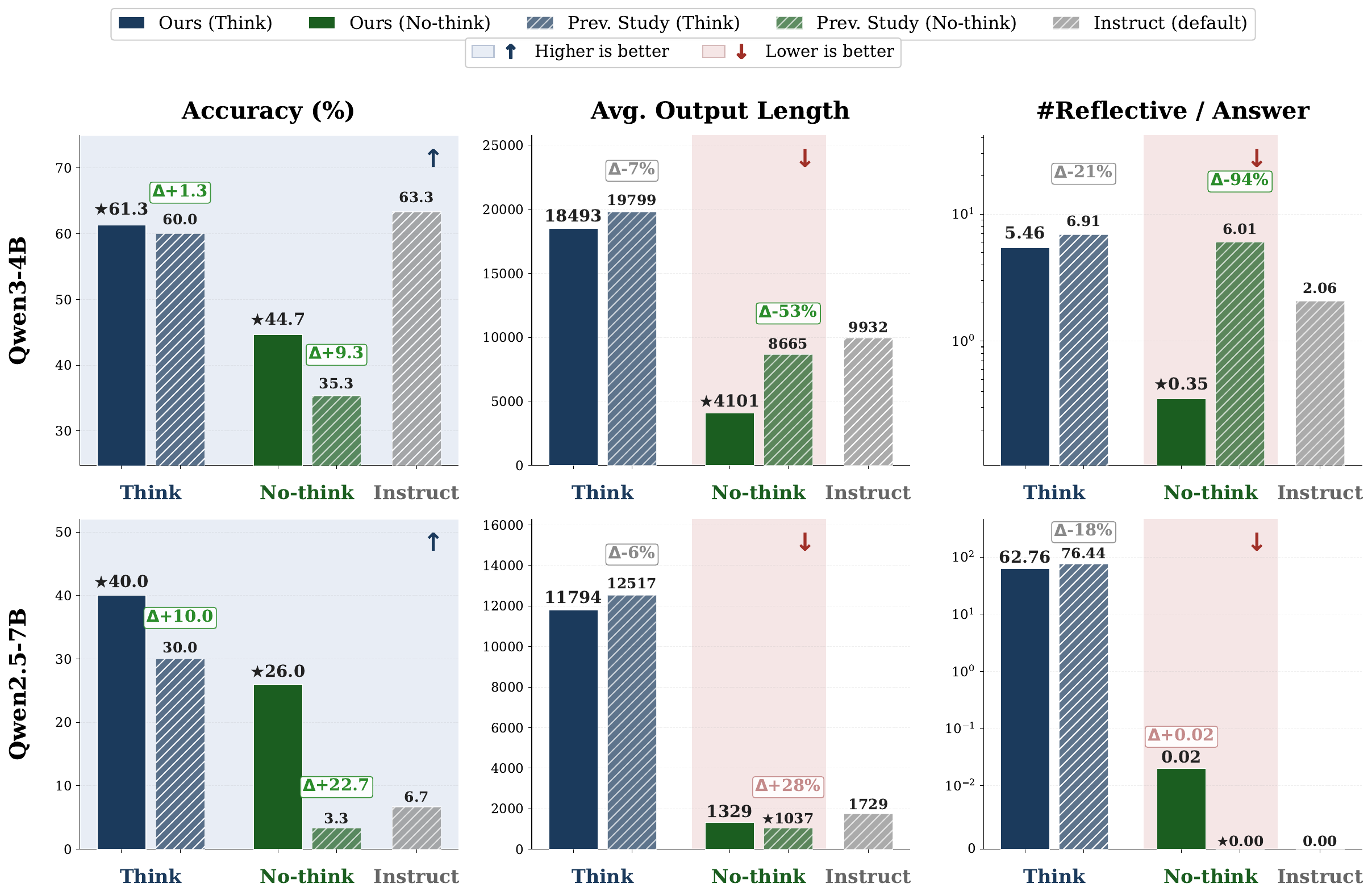}
    \caption{Accuracy, average output length, and per-answer reflective token count on AIME24 for two \ple{} instantiations (Qwen2.5-7B, top; Qwen3-4B, bottom) compared with baselines. ``Ours'' denotes \ple{} models built on each backbone; ``SFT-only'' is an SFT-only baseline in previous studies~\citep{wang2025demystifying}; ``Instruct'' is the off-the-shelf instruction-tuned backbone, shown under its single default-mode response. On Qwen2.5-7B, \ple{} matches the near-zero reflective token count of the SFT-only baseline despite much higher \nothink{} accuracy. On Qwen3-4B, \ple{} achieves substantially stronger mode separation than SFT-only: reasoning leakage drops by \textbf{17$\times$} and \nothink{} accuracy improves by \textbf{+9.3} points, while \think{}-mode performance is maintained.}
\label{fig:exp_aime24}
\end{figure}
\subsection{Experimental setup}
\label{sec:exp_setup}
\textbf{Dataset Construction}
We build our training corpus from the \textit{Superior-Reasoning}\citep{yan2026distribution} problem set, pairing its reasoning traces with \nothink{} targets synthesized by Qwen3-235B in \nothink{} mode and retained only if they match the ground-truth answer, respect domain-specific length limits, and contain no reflective tokens. The corpus holds \textbf{54k} examples: 27k \think{} traces and 27k direct answers, a 1:1 ratio. Because the \nothink{} targets are distilled from a stronger teacher, one might ask whether \ple{}'s gains reflect teacher-style imitation rather than architecture; \ple{} and its dense baselines train on identical data, so any such bias acts on both equally, and Appendix~\ref{sec:appendix_data_bias} rules it out with a controlled data-construction ablation.

\textbf{Backbone weights.}
We evaluate \ple{} across two representative model families to demonstrate its generalizability. For both families we initialize from the corresponding \textit{Instruct} checkpoint, keeping the two families structurally symmetric: \textit{Qwen2.5-7B-Instruct} for the Qwen2.5 family, and \textit{Qwen3-4B-Instruct} for the Qwen3 family. For each backbone, we initialize the \ple{} architecture by cloning their MLP weights, and fine-tune on our training corpus constructed from the \textit{Superior-Reasoning} dataset~\citep{yan2026distribution}.

\textbf{Baselines.}
We compare \ple{} (\textit{Ours}) against two baselines per family: (1) \textit{Instruct}, the off-the-shelf backbone; and (2) \textit{SFT-only}, the SFT-only recipe~\citep{wang2025demystifying}. All three models start from the same Instruct checkpoint. Results with additional backbone initializations (\textit{Qwen3-4B} and \textit{Qwen3-4B-Base}) are presented in Appendix~\ref{sec:appendix_original_qwen3_hybrid} and Appendix~\ref{sec:appendix_abl_results_base}.

\textbf{Benchmarks and Metrics.}
We evaluate all models in both \think{} and \nothink{} modes across four reasoning benchmarks: MATH500~\citep{lightman2023lets}, AIME24~\citep{aime2024}, MMLU-STEM~\citep{hendryckstest2021}, and GPQA-Diamond~\citep{rein2024gpqa}. To measure performance and mode separation, we report: (1) \textit{Accuracy (\%)} for task proficiency; (2) \textit{Average output length} to monitor verbosity; and (3) \textit{Reflective tokens per answer (\#Refl./Ans.)}—the mean count of self-reflective markers (e.g., \texttt{wait}, \texttt{hmm}, \texttt{alternatively}). As our two metric for reasoning leakage, average output length and \#Refl./Ans. are complementary rather than redundant: average output length measures verbosity, while \#Refl./Ans.\ measures whether reflective-reasoning tokens are specifically present. We validate these metrics track genuine semantic leakage rather than incidental token counts using a model-based judge (GPT-OSS-20B) applied to off-the-shelf hybrid-thinking baselines across three model sizes and three benchmarks (Appendix~\ref{sec:appendix_judge_validation}): judge verdicts agree with \#Refl./Ans.\ in 92--100\% of cases in both directions.

\subsection{Main Results}
\label{sec:main_results}

We present a comprehensive evaluation of \ple{} across two model families to verify its effectiveness in mitigating reasoning leakage while maintaining task performance. Overall, our results demonstrate that \ple{} provides a superior Pareto frontier for hybrid thinking: it delivers a \nothink{} mode that is cleaner and more proficient than prior training-level recipes, while simultaneously safeguarding (or even enhancing) the model's capacity for deep, multi-step reasoning. Our primary analysis focuses on the challenging \textit{AIME24} benchmark (\Cref{fig:exp_aime24}), which serves as a stress test for hybrid thinking due to its requirement for deep deliberation; results for \textit{MATH500} are consistent with these findings and are provided in \Cref{fig:exp_math500} (with additional details in Appendix~\ref{sec:appendix_exp_results}).

\begin{table}[t]
    \centering
    \caption{\textbf{Reasoning leakage and accuracy in the \nothink{} (direct-answer) mode.} \#Refl./Ans.: mean reflective token count per response (e.g., \texttt{wait}, \texttt{hmm}). \textbf{Bold} marks \textit{PLE (Ours)}. \textit{Hybrid} and \textit{Instruct} report each off-the-shelf model's own default/\nothink{} behavior; coming from different post-training lineages, their accuracies are not directly comparable across lineages---the controlled comparison is between \textit{SFT-only} and \textit{PLE (Ours)}, both initialized from the same \textit{Instruct} checkpoint. \textit{SFT-only}: the two-phase SFT recipe of \citet{wang2025demystifying} on the same backbone. \think{}-mode performance is preserved (\Cref{fig:exp_aime24,fig:exp_math500,fig:exp_composite_gpqa,fig:exp_composite_mmlu_stem}).}
    \vspace{-8pt}
    \label{tab:exp_leakage}
    \resizebox{\textwidth}{!}{
    \begin{tabular}{ccccccccc}
    \toprule
    \multirow{2}{*}{Model} & \multirow{2}{*}{Type}
    & \multicolumn{3}{c}{AIME24} & \multicolumn{3}{c}{MATH500} \\
    \cmidrule(lr){3-5} \cmidrule(lr){6-8}
    & & Acc. & Len. & \#Refl./Ans. & Acc. & Len. & \#Refl./Ans. \\
    \midrule
    \multirow{4}{*}{Qwen3-4B}
     & Hybrid       & 23.33 & 12506 & 0.17 & 82.60 & 1583 & 0.04 \\
     & Instruct     & 63.33 & 9932 & 2.06 & 94.52 & 1812.55 & 0.36 \\
     & SFT-only     & 35.33 & 8665 & 6.01 & 83.28 & 1040.37 & 0.15 \\
     & \textbf{PLE (Ours)} & \textbf{44.67} & \textbf{4101} & \textbf{0.35} & \textbf{82.36} & \textbf{702.23} & \textbf{0.05} \\
    \midrule
    \multirow{3}{*}{Qwen2.5-7B}
     & Instruct     & 6.67  & 1729 & 0    & 59.94 & 703.11 & 0    \\
     & SFT-only     & 3.33  & 1037 & 0.00 & 63.56 & 593.09 & 0.31 \\
     & \textbf{PLE (Ours)} & \textbf{26.00} & \textbf{1329} & \textbf{0.02} & \textbf{83.20} & \textbf{614.46} & \textbf{0.01} \\
    \bottomrule
    \end{tabular}
    }

    \vspace{6pt}
    \resizebox{\textwidth}{!}{
    \begin{tabular}{ccccccccc}
    \toprule
    \multirow{2}{*}{Model} & \multirow{2}{*}{Type}
    & \multicolumn{3}{c}{GPQA-Diamond} & \multicolumn{3}{c}{MMLU-STEM} \\
    \cmidrule(lr){3-5} \cmidrule(lr){6-8}
    & & Acc. & Len. & \#Refl./Ans. & Acc. & Len. & \#Refl./Ans. \\
    \midrule
    \multirow{4}{*}{Qwen3-4B}
     & Hybrid       & 46.97 & 2679 & 0.05 & 85.95 & 483.79 & 0.01 \\
     & Instruct     & 42.32 & 445.08 & 0.14 & 87.54 & 295.95 & 0.05 \\
     & SFT-only     & 45.96 & 2513.78 & 5.34 & 86.71 & 263.92 & 0.06 \\
     & \textbf{PLE (Ours)} & \textbf{45.86} & \textbf{736.34} & \textbf{0.23} & \textbf{86.61} & \textbf{180.19} & \textbf{0.02} \\
    \midrule
    \multirow{3}{*}{Qwen2.5-7B}
     & Instruct     & 30.15 & 775  & 0    & 62.00 & 211 & 0    \\
     & SFT-only     & 30.00 & 2910 & 3.22 & 59.44 & 774 & 0.83 \\
     & \textbf{PLE (Ours)} & \textbf{30.00} & \textbf{3857} & \textbf{0.00} & \textbf{92.22} & \textbf{139} & \textbf{0.00} \\
    \bottomrule
    \end{tabular}
    }
\end{table}

\paragraph{Effective Mode Separation with Minimal Leakage.}
\Cref{tab:exp_leakage} reports \nothink{}-mode behavior for every model; we begin with mode separation.
On Qwen3-4B, \ple{} cuts reasoning leakage \textbf{17$\times$} below the SFT-only baseline (\textbf{6.01} to \textbf{0.35} reflective tokens per answer) and raises \nothink{} accuracy by \textbf{+9.3} points (35.33\% to 44.67\%), while \think{}-mode performance is maintained (61.33\% vs.\ 60.00\%). SFT-only is only partially effective: it halves output length between modes (19{,}799 to 8{,}665 tokens) but leaves reflective content untouched (6.91 to 6.01), whereas \ple{} halves length to 4{,}101 tokens and removes the reflective content (\Cref{fig:exp_aime24}).

\paragraph{Significant Enhancement of \nothink{} Proficiency.}

\ple{}'s \nothink{} mode is not merely cleaner but more capable, consistently across both families. On Qwen3-4B it answers at \textbf{44.67\%}---\textbf{+9.3} points over the SFT-only recipe---in \textbf{4{,}101} tokens at \textbf{0.35} reflective tokens, and the margin is larger on Qwen2.5-7B (\textbf{26.00\%} vs.\ 3.33\%, \textbf{+22.7} points) at 0.02 reflective tokens (\Cref{tab:exp_leakage}). These gains are not won on easy ground: despite its name, \textit{Qwen3-4B-Instruct} is a non-thinking release of a hybrid-thinking base~\citep{yang2025qwen3,qwen2507card2026}, reasoning at length with no direct-answer path (Appendix~\ref{sec:appendix_instruct_probe}). From that the SFT-only recipe supplies a path but leaks through it (8{,}665 tokens at 6.01); \ple{} makes it short and accurate.

\paragraph{Preservation and Synergy of Reasoning Capabilities.}
A vital requirement for \ple{} is that architectural isolation must not degrade the model's core reasoning power. Our results confirm that the \think{} mode's performance remains robust across both model families. Interestingly, we observe a synergistic effect in the Qwen2.5-7B-Instruct family: the \think{} mode of \ple{} actually \textbf{outperforms} the original model on AIME24 accuracy (\Cref{fig:exp_aime24}). This indicates that sharing the attention backbone preserves the underlying logic while the dual-expert MLP structure lets the \think{} mode focus on deliberation.

\paragraph{Generalization Beyond Qwen and Math/STEM.}
To test whether \ple{}'s mode separation is specific to the Qwen family or to mathematical reasoning, we extend our evaluation to a structurally different model family, \textit{Llama-3.1-8B}, and to a broad general-knowledge benchmark, full \textit{MMLU}. Mode separation holds throughout---on Llama-3.1-8B, \ple{}'s \nothink{} mode is the most accurate no-think model on all four benchmarks tested, most strikingly on MMLU-STEM (76.7\% vs.\ 53.0--53.4\%)---so \ple{} spans three model families and five benchmark types; see Appendix~\ref{sec:appendix_generalization} for full results.

\section{Ablation Studies}\label{sec:ablation}

Beyond the main evaluations in \Cref{sec:experiments}, we conduct a series of ablation studies to dissect the impact of various design choices on \ple{}'s performance and mode separation behavior. Our analysis reveals that two primary factors significantly dictate the model's efficacy: the selection of \textbf{backbone weights} for initialization and the composition of the \textbf{training dataset}. In the following subsections, we explore how these factors influence the delicate balance between reasoning power and leakage mitigation.

\subsection{Effect of backbone initialization}\label{sec:ablation_base}

The choice of which pretrained model to use as the backbone for \ple{} is a critical design decision.
We compare four backbones: 1) \textit{Qwen2.5-7B-Instruct}, 2) \textit{Qwen3-4B-Instruct}, the backbone used for our main-text Qwen3 results (\Cref{sec:experiments}), 3) \textit{Qwen3-4B}, and 4) \textit{Qwen3-4B-Base} (the pretrained model without any post-training).

All variants are trained with identical hyperparameters to isolate the effect of backbone weights. \Cref{fig:ablation_base_aime24} provides the results on the AIME24 task; more comprehensive experimental results are provided in Appendix~\ref{sec:appendix_abl_results_base}.

\begin{figure}[t]
  \centering
  \includegraphics[width=1.0\linewidth]{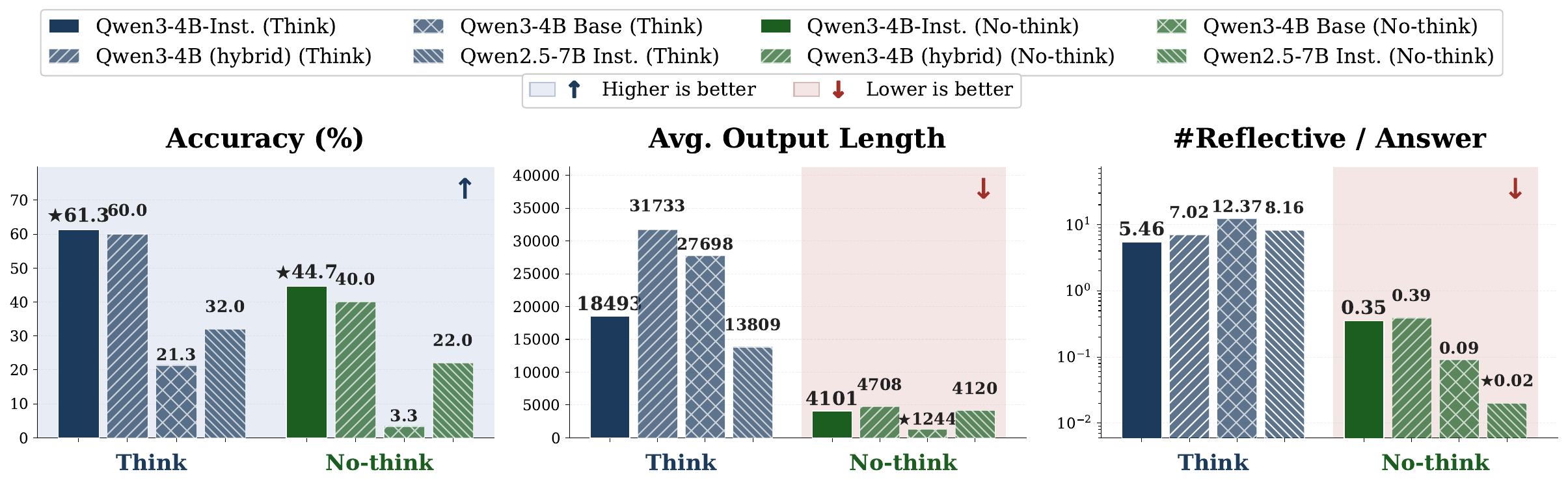}
  \caption{Ablation on backbone choice (AIME24). PLE is initialized from ``Qwen2.5-7B Instruct'' (pure instruction-tuned weights), ``Qwen3-4B Instruct'', ``Qwen3-4B Base'' (raw pretrained weights), or ``Qwen3-4B'' (hybrid-thinking weights). Residual \nothink{} leakage orders monotonically with each initialization's own reasoning behavior on both dimensions (reflective tokens/answer 0.02 $<$ 0.35 $<$ 0.39, with \nothink{} output length in the same order), at some cost in accuracy. The pretrained base collapses on this high-difficulty task, confirming that PLE requires a backbone with existing instruction-following capabilities.}
  \label{fig:ablation_base_aime24}
\end{figure}

\textbf{Comparison between Three Backbones: Effect of Retained Latent Reasoning Behavior.}
Comparing \textit{Qwen3-4B} (a hybrid-thinking model) with \textit{Qwen2.5-7B-Instruct} (a pure instruction-tuned model) reveals a clear divergence in behavior. We find that using a pure \textit{Instruct} model as the base generally yields superior mode separation, characterized by a ``cleaner'' \nothink{} mode with virtually zero leakage. In contrast, using the hybrid \textit{Qwen3-4B} leads to higher absolute performance (accuracy) but carries a persistent residue of reasoning artifacts.

We hypothesize that hybrid-thinking models carry reasoning-heavy patterns from their intensive post-training. This residual reasoning behavior is deeply embedded in the shared attention representations, making it difficult to fully extinguish reflective behaviors in the \nothink{} mode, even with architectural isolation. However, because the hybrid base already possesses a stronger foundation for the \think{} mode, it ultimately achieves higher task scores, suggesting that the initial ``reasoning literacy'' of the weights sets a higher ceiling for performance.
Not every instruction-tuned checkpoint is equally free of these patterns: \textit{Qwen3-4B-Instruct} is derived from a hybrid-thinking base and released as a non-thinking update within the same family~\citep{yang2025qwen3,qwen2507card2026}. Accordingly, residual \nothink{} leakage after \ple{} fine-tuning orders with each backbone's own reasoning behavior on both dimensions: reflective tokens per answer run 0.02 $<$ 0.35 $<$ 0.39, and AIME24 \nothink{} length follows the same order, 1{,}329 $<$ 4{,}101 $<$ 5{,}597 tokens.

\textbf{The Necessity of Post-training: Qwen3-4B vs. Qwen3-Base.}
To isolate the impact of post-training, we evaluated \textit{Qwen3-4B-Base} (the raw pretrained weights). While this base shows reasonable performance on simpler benchmarks like MATH500, its performance collapses on the high-difficulty \textit{AIME24} tasks. In many cases, the base-initialized \ple{} fails to converge on a coherent reasoning strategy, often exceeding token limits without reaching a solution. This result underscores that \ple{} requires a backbone that has already acquired basic instruction-following and logical structuring capabilities. Relying solely on the PLE to ``learn'' reasoning from scratch is insufficient for complex domains.

\subsection{Effect of training dataset}\label{sec:ablation_dataset}
We investigate the impact of dataset quality and composition on \ple{} training.
Specifically, we compare the \textit{OpenR1} dataset used in \citet{wang2025demystifying}, constructed from OpenR1-Math~\citep{openr1} and \textit{Superior-reasoning}: our improved dataset featuring harder reasoning questions and higher-quality responses. 
\begin{figure}[t]
    \centering
    \includegraphics[width=1.0\linewidth]{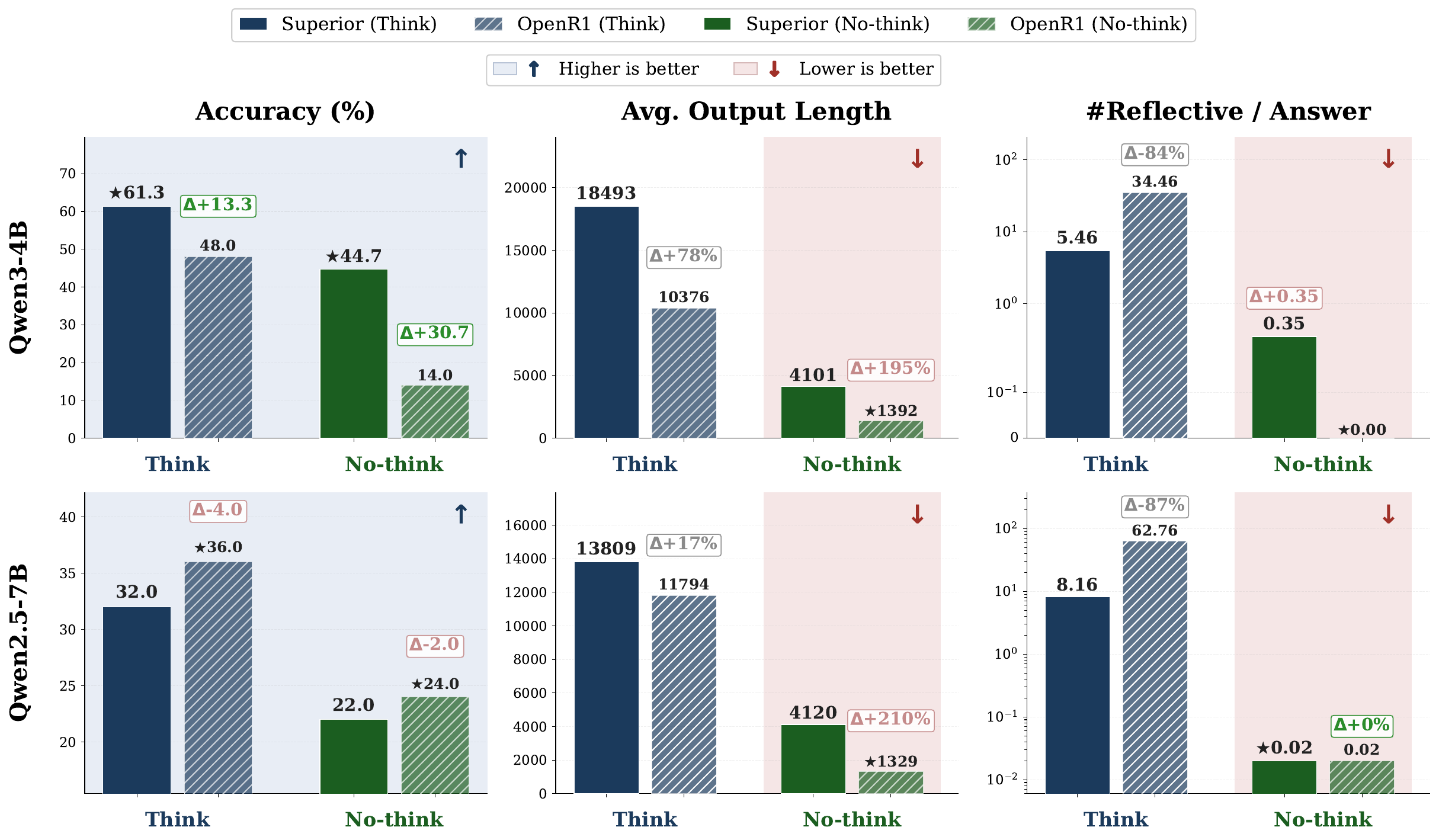}
    \caption{Ablation on training dataset (AIME24) for two backbones: \textit{Qwen3-4B-Instruct} (top) and \textit{Qwen2.5-7B-Instruct} (bottom). ``Superior'': \ple{} trained on our high-difficulty dataset with longer CoT traces; ``OpenR1'': on simpler data. On Qwen3-4B-Instruct, Superior wins in both modes (61.33\% vs.\ 48.00\% \think{}; 44.67\% vs.\ 14.00\% \nothink{}) and reasons more efficiently (5.46 vs.\ 34.46 reflective tokens in \think{}). On Qwen2.5-7B, OpenR1 remains competitive---dataset difficulty must match backbone capacity.}
    \label{fig:ablation_dataset_aime24_inst}
\end{figure}
Our findings are summarized in \Cref{fig:ablation_dataset_aime24_inst} and \Cref{tab:ablation_dataset_by_base_weight} (Appendix~\ref{sec:appendix_abl_results_dataset}), and point to a trade-off between proficiency and mode purity.

We investigate the impact of dataset quality and difficulty on \ple{} training by comparing two distinct sources: the \textbf{OpenR1} dataset (a relatively simpler collection) and our \textbf{Superior-reasoning} dataset (higher difficulty and longer reasoning chains). Our analysis, supported by \Cref{tab:ablation_dataset_by_base_weight} and \Cref{fig:ablation_dataset_aime24_inst}, reveals how data composition interacts with architectural experts.

\textbf{Difficulty-Capacity Alignment.}
A key observation from our ablation is that the effectiveness of a dataset is closely tied to the inherent capability of the backbone. As shown in \Cref{tab:ablation_dataset_by_base_weight}, the simpler \textit{OpenR1} dataset yields better relative gains when paired with the \textit{Qwen2.5-7B-Instruct} backbone. Conversely, the more challenging \textit{Superior-reasoning} dataset excels when paired with the more capable \textit{Qwen3-4B} model. 

For the Qwen3-4B backbone, the advantage of the superior dataset is particularly pronounced in the \nothink{} mode, where it significantly outperforms the OpenR1 baseline in accuracy. This suggests that as backbones become more proficient, they require higher-quality, ``harder'' data to fully specialize their dual experts. We conclude that dataset ``difficulty'' must be aligned with the backbone's capacity to achieve optimal results.

\textbf{The Leakage-Performance Trade-off.}
While the \textit{Superior-reasoning} dataset provides clear performance benefits, it also induces more noticeable reasoning leakage compared to OpenR1. We attribute this to the nature of the data: the superior dataset contains significantly longer and more complex chain-of-thought (CoT) traces. These high-intensity reasoning patterns appear to be more ``contagious,'' exerting greater pressure on the architectural separation and making the \nothink{} expert more prone to emitting reflective artifacts. 

This represents a fundamental trade-off in hybrid thinking: achieving the highest levels of reasoning proficiency often requires sacrificing a degree of mode-control purity. In most practical scenarios, we argue that this trade-off is acceptable, as the gains in \nothink{} accuracy and \think{} depth outweigh the presence of minor leakage artifacts.

\subsection{Isolating the architectural contribution from added parameters}\label{sec:ablation_specialists}
Do \ple{}'s gains come from deterministic routing, or merely from extra parameters? From an identical \textit{Qwen3-4B-Base} checkpoint we train a \textit{think-only} and a \textit{no\_think-only} specialist on matched data, and compare both against \textit{Base-PLE} trained on the combined data and against a naive weight-merge of the two specialists. Across MATH500, AIME24, and GPQA-Diamond, \textit{Base-PLE} matches or exceeds both dedicated specialists within one model, while the weight-merge collapses accuracy on AIME24 and MATH500 and still leaks reasoning in \nothink{} mode (\Cref{tab:specialists_summary}; full tables in Appendix~\ref{sec:appendix_specialists}).

An \textbf{initialization-equivalence test} confirms the source: before fine-tuning, \ple{} is token-for-token identical to the backbone on 12/12 prompt--mode pairs and all 108/108 tensors match numerically, so the extra expert contributes nothing until training decouples the two objectives. Relative to deploying two independent models, \textbf{\ple{} saves up to 33\%} of a single model's parameters ($\approx$1.67$\times$ one model in total), at single-dense-model inference cost (\Cref{tab:appendix_routing_params}). Appendix~\ref{sec:appendix_data_bias} separately rules out teacher-style bias from the distilled \nothink{} targets.

\section{Related Work}\label{sec:related_work}

\textbf{LLM Reasoning.}
Reasoning in LLMs has been advanced through RL~\citep{guo2025deepseek, jaech2024openai, team2025kimi, shao2024deepseekmath, xie2025logicrlunleashingllmreasoning, zeng2025simplerlzooinvestigatingtamingzero} and SFT~\citep{li2025llmseasilylearnreason, muennighoff2025s1simpletesttimescaling, ye2025limoreasoning, ji2025first}. Our work is orthogonal: rather than improving reasoning quality, we focus on \emph{controlling when reasoning occurs} through architectural design.

\textbf{Hybrid Thinking and Efficient Reasoning.}
Hybrid thinking~\citep{sui2025stop} lets models switch between chain-of-thought and direct answering via control tokens, and has been adopted by several recent systems~\citep{gemini25, yang2025qwen3, agarwal2025gpt, deepseekv31}. However, reasoning behaviors still leak into \nothink{} mode. \citet{wang2025demystifying} showed that training-level interventions---data scale, ratio tuning, multi-stage schedules---reduce but cannot eliminate this leakage. Other efficiency-oriented approaches include CoT compression~\citep{xia2025tokenskip, zhang2025lightthinker, xu2025chain, aytes2025sketch}, early stopping~\citep{fu2024efficiently}, and preference optimization~\citep{reduce_overthinking_2025}. Unlike all of these, we address leakage at the \emph{architectural level}.

\textbf{Mixture-of-Experts.}
MoE architectures~\citep{shazeer2017outrageously, fedus2022switch, lepikhin2021gshard, jiang2024mixtral, dai2024deepseekmoe} use learned token-level routers for conditional computation, but face challenges including load imbalance, training instability, and the need for auxiliary balancing losses. Most relevant to our setting, Metis-HOME~\citep{lan2025metishome} applies a two-expert MoE with a learned router to hybrid thinking, differing from \ple{} in routing (a trainable gate vs.\ our parameter-free control-token lookup) and in modality (vision-language vs.\ text-only). Substituting a learned router into our setting would change the research question from ``can deterministic routing achieve clean mode separation?'' to ``which router design wins?''---a different paper. We compare the design space in Appendix~\ref{sec:appendix_related_work_extended} (\Cref{tab:moe_comparison}).

\section{Conclusion}\label{sec:conclusion}
In this paper, we argued that reasoning leakage in hybrid-thinking models is partly an architectural problem. We introduced Path-Lock Expert (PLE), a minimal modification that replaces each decoder MLP with two semantically locked experts routed deterministically by control tokens. Across math and science benchmarks on two model families, PLE substantially strengthens the \nothink{} mode---reducing reflective tokens, shortening outputs, and improving accuracy---while preserving \think{} performance. Our ablations further show that PLE's effectiveness depends on backbone initialization and dataset--capacity alignment. These results demonstrate that separating mode-specific feed-forward pathways is a straightforward and effective complement to training-level interventions for achieving more controllable hybrid thinking models.

\section*{Acknowledgments}
This work was supported in part by NSF awards 2117439 and 2112606.

The first author would like to express his sincere gratitude to Yanbo Zhao, whose guidance and
influence inspired him to pursue computer science and AI research, and whose care and support have
accompanied him throughout this journey. He would also like to thank Kaijun Tan for his continued
support and for generously sharing his practical advice and valuable insights on LLM engineering.

\section*{LLM Usage Statement}
The authors used LLM-based assistants to help with LaTeX editing, polishing writing, proofreading, and code implementation; all content was reviewed and verified by the authors.

\bibliography{colm2026_conference}
\bibliographystyle{colm2026_conference}

\newpage
\appendix
\section{Appendix}
\subsection{Supplemental Results for Main Experiment Results}
\label{sec:appendix_exp_results}
\begin{figure}[h!]
    \centering
    \includegraphics[width=1.0\linewidth]{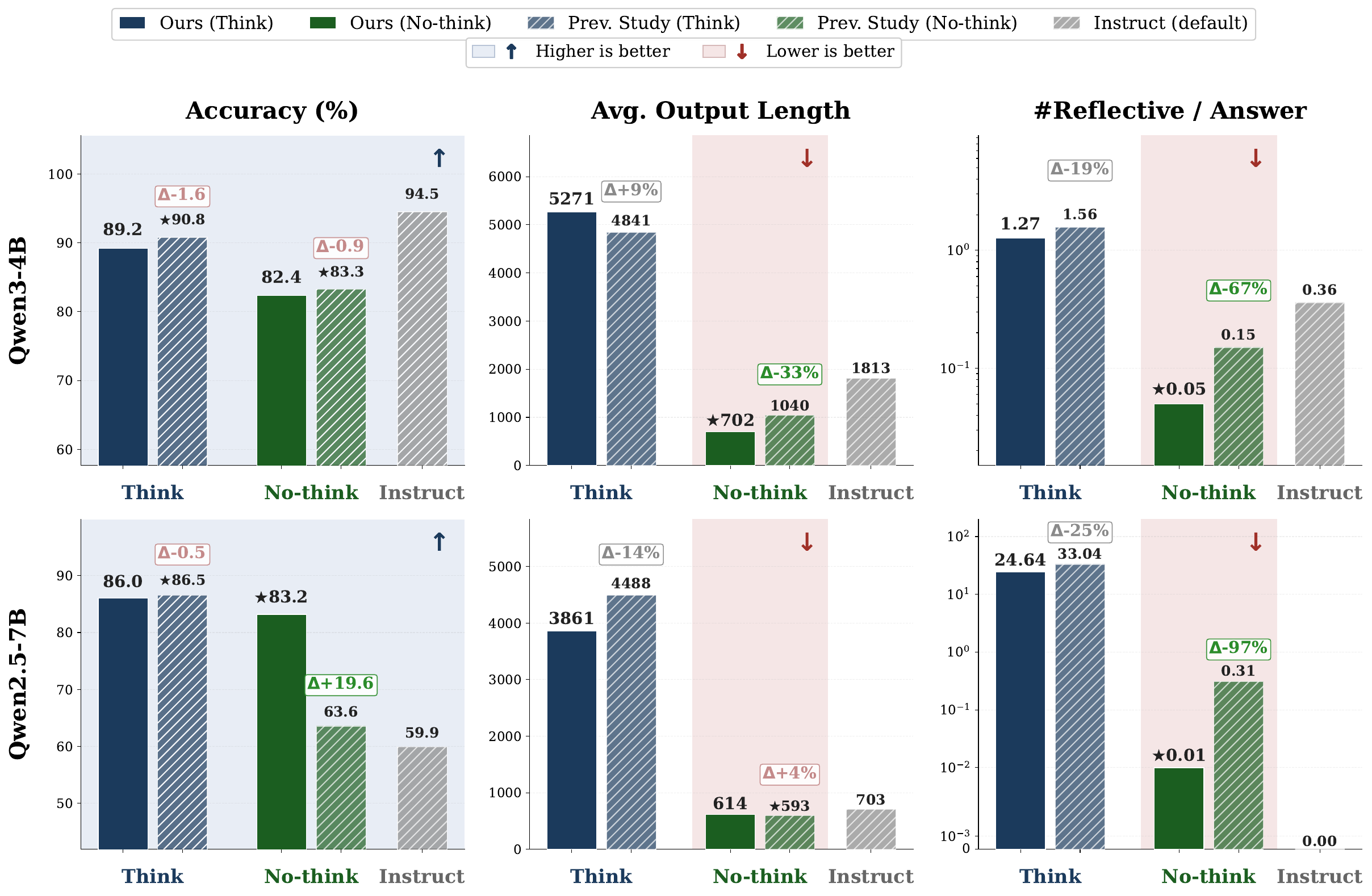}
    \caption{Accuracy, average output length, and per-answer reflective token count on MATH500 for two \ple{} instantiations (Qwen2.5-7B, top; Qwen3-4B, bottom) compared with baselines, in the same format as \Cref{fig:exp_aime24}. ``Instruct'' is shown under its single default-mode response (no control-token suffix appended).}
\label{fig:exp_math500}
\end{figure}

\begin{table}[h!]
    \centering
    \caption{Reasoning leakage comparison on \textbf{MMLU-STEM} and \textbf{MMLU-full}. Same setup and notation as \Cref{tab:exp_leakage}.}
    \vspace{-8pt}
    \label{tab:exp_leakage_mmlu_gpqa}
    \resizebox{\textwidth}{!}{
    \begin{tabular}{ccccccccc}
    \toprule
    \multirow{2}{*}{Model} & \multirow{2}{*}{Type} & \multirow{2}{*}{Mode}
    & \multicolumn{3}{c}{MMLU-STEM} & \multicolumn{3}{c}{MMLU-full} \\
    \cmidrule(lr){4-6} \cmidrule(lr){7-9}
    & & & Acc. & Len. & \#Refl./Ans. & Acc. & Len. & \#Refl./Ans. \\
    \midrule
    \multirow{6}{*}{Qwen2.5-7B}
     & \multirow{2}{*}{Instruct}      & Think    & N/A   & N/A  & N/A  & N/A   & N/A  & N/A  \\
     &                                 & No-think & 62.00 & 211  & 0    & 73.82 & 275.01  & 0.0000    \\
    \cmidrule(lr){2-9}
     & \multirow{2}{*}{SFT-only}      & Think    & 85.82 & 3100 & 24.74 & -- & --  & -- \\
     &                                 & No-think & 59.44 & 774  & 0.83  & -- & --  & --  \\
    \cmidrule(lr){2-9}
     & \multirow{2}{*}{Ours}          & Think    & 88.89 & 429  & 0     & -- & --  & -- \\
     &                                 & No-think & 92.22 & 139  & \textbf{0.00} & 76.28 & 357.27 & \textbf{0.0000} \\
    \midrule
    \multirow{6}{*}{Qwen3-4B}
     & \multirow{2}{*}{Instruct}       & Think    & N/A   & N/A     & N/A   & N/A   & N/A      & N/A   \\
     &                                & No-think & 87.54 & 295.95 & 0.05 & 79.79 & 264.88 & 0.02 \\
    \cmidrule(lr){2-9}
     & \multirow{2}{*}{SFT-only}       & Think    & 90.96 & 1491.51 & 0.22 & 78.43 & 2777.88 & 2.69 \\
     &                                & No-think & 86.71 & 263.92  & 0.06 & 75.58 & 188.76  & 0.16  \\
    \cmidrule(lr){2-9}
     & \multirow{2}{*}{Ours}          & Think    & 91.06 & 1606.65 & 0.49 & 78.17 & 3175.23 & 2.85 \\
     &                                & No-think & 86.61 & 180.19  & \textbf{0.02} & 75.60 & 117.18 & \textbf{0.00} \\
    \bottomrule
    \end{tabular}
    }
\end{table}

\subsection{Results with the Native Hybrid Qwen3-4B Backbone}
\label{sec:appendix_original_qwen3_hybrid}

These tables report \ple{} initialized from the native hybrid-thinking \textit{Qwen3-4B} checkpoint, rather than from the \textit{Qwen3-4B-Instruct} checkpoint used for the Qwen3 family in the main text (\Cref{sec:exp_setup}). Initializing both families from an Instruct checkpoint keeps the main-text comparison structurally symmetric; the hybrid-initialized variant is reported here for reference.

\begin{table}[h!]
    \centering
    \caption{Results with the native hybrid-thinking \textit{Qwen3-4B} backbone on MATH500 and AIME24. ``Hybrid'' is the off-the-shelf checkpoint; ``Ours'' is \ple{} initialized from it. No SFT-only baseline was run on this backbone.}
    \vspace{-8pt}
    \label{tab:appendix_original_qwen3_hybrid_math_aime}
    \resizebox{\textwidth}{!}{
    \begin{tabular}{cccccccc}
    \toprule
    \multirow{2}{*}{Type} & \multirow{2}{*}{Mode}
    & \multicolumn{3}{c}{MATH500} & \multicolumn{3}{c}{AIME24} \\
    \cmidrule(lr){3-5} \cmidrule(lr){6-8}
    & & Acc. & Len. & \#Relf./Ans. & Acc. & Len. & \#Relf./Ans.\\
    \midrule
    \multirow{2}{*}{Hybrid}       & Think    & 92.02 & 4679.54 & 19.60 & 61.67 & 11595.51 & 45.85 \\
                                & No-think & 82.22 & 1004.13 & 0.12  & 20.67 & 4636.37  & 2.54  \\
    \midrule
    \multirow{2}{*}{Ours}         & Think    & 94.8   &  6050.45  & 3.56  & 60.00 & 31733.18 & 7.02 \\
                                & No-think & 80.80 & 676.78 & 0.13 & 40.0 & 5597.0 & 0.39 \\
    \bottomrule
    \end{tabular}
    }
\end{table}

\begin{table}[h!]
    \centering
    \caption{Results with the native hybrid-thinking \textit{Qwen3-4B} backbone on MMLU-STEM and GPQA-Diamond. Same setup as \Cref{tab:appendix_original_qwen3_hybrid_math_aime}.}
    \vspace{-8pt}
    \label{tab:appendix_original_qwen3_hybrid_mmlu_gpqa}
    \resizebox{\textwidth}{!}{
    \begin{tabular}{cccccccc}
    \toprule
    \multirow{2}{*}{Type} & \multirow{2}{*}{Mode}
    & \multicolumn{3}{c}{MMLU-STEM} & \multicolumn{3}{c}{GPQA-Diamond} \\
    \cmidrule(lr){3-5} \cmidrule(lr){6-8}
    & & Acc. & Len. & \#Refl./Ans. & Acc. & Len. & \#Refl./Ans. \\
    \midrule
    \multirow{2}{*}{Hybrid}        & Think    & 90.26 & 2338 & 16.82 & 53.13 & 7615  & 67.01 \\
                                 & No-think & 84.49 & 612  & 0.01  & 42.83 & 1504  & 0.08  \\
    \midrule
    \multirow{2}{*}{Ours}          & Think    & 92.22 & 1041 & 0     & 48.00 & 25741 & 5.14  \\
                                 & No-think & 83.70 & 253  & 0.01 & 40.00 & 402  & 0.00 \\
    \bottomrule
    \end{tabular}
    }
\end{table}

\begin{figure}[t]
    \centering
    \includegraphics[width=1.0\linewidth]{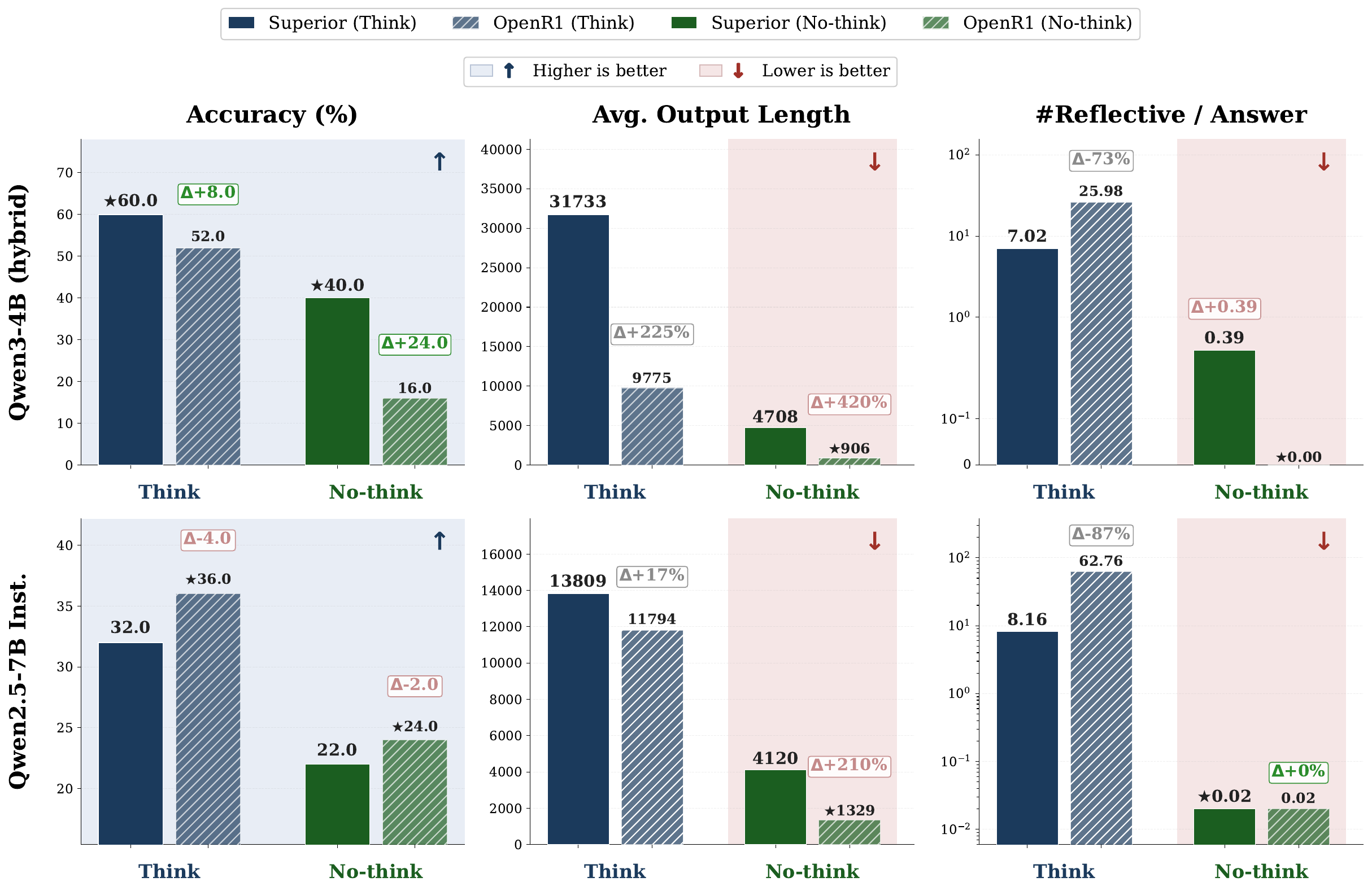}
    \caption{Ablation on training dataset (AIME24) with the native hybrid-thinking \textit{Qwen3-4B} backbone. ``Superior'': PLE trained on our high-difficulty dataset with longer CoT traces; ``OpenR1'': trained on simpler data. Results for Qwen3-4B (top) and Qwen2.5-7B (bottom). On the stronger Qwen3-4B backbone, Superior yields higher accuracy and fewer \think{} reflective tokens (7.02 vs.\ 26), suggesting higher-quality data enables more efficient reasoning. On the weaker Qwen2.5-7B, OpenR1 performs comparably, indicating dataset difficulty must align with backbone capacity. Both maintain near-zero \nothink{} leakage, though Superior produces slightly more reflective artifacts (0.39 vs.\ 0 on Qwen3-4B), illustrating a leakage--performance trade-off.}
    \label{fig:ablation_dataset_aime24}
\end{figure}

\subsection{Additional Main-Results Figures}
\label{sec:appendix_additional_figures}

Composite accuracy/length/reflective-token figures for the remaining three benchmarks not already shown in \Cref{fig:exp_aime24} (AIME24, main text) or \Cref{fig:exp_math500} (MATH500, above), in the same format. ``Instruct'' rows in all three figures show each model's single mode response (no control-token suffix appended).

\begin{figure}[h!]
    \centering
    \includegraphics[width=1.0\linewidth]{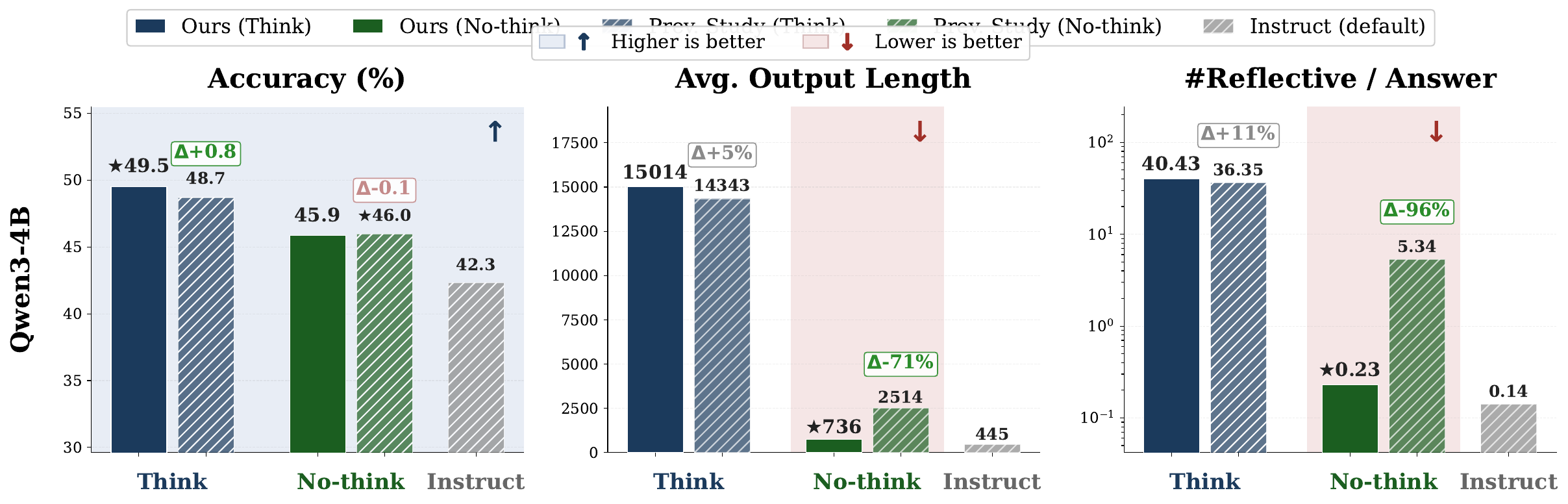}
    \caption{Accuracy, average output length, and per-answer reflective token count on GPQA-Diamond for two \ple{} instantiations (Qwen2.5-7B, top; Qwen3-4B, bottom) compared with baselines, in the same format as \Cref{fig:exp_aime24}. ``Instruct'' is shown under its single default-mode response (no control-token suffix appended).}
    \label{fig:exp_composite_gpqa}
\end{figure}

\begin{figure}[h!]
    \centering
    \includegraphics[width=1.0\linewidth]{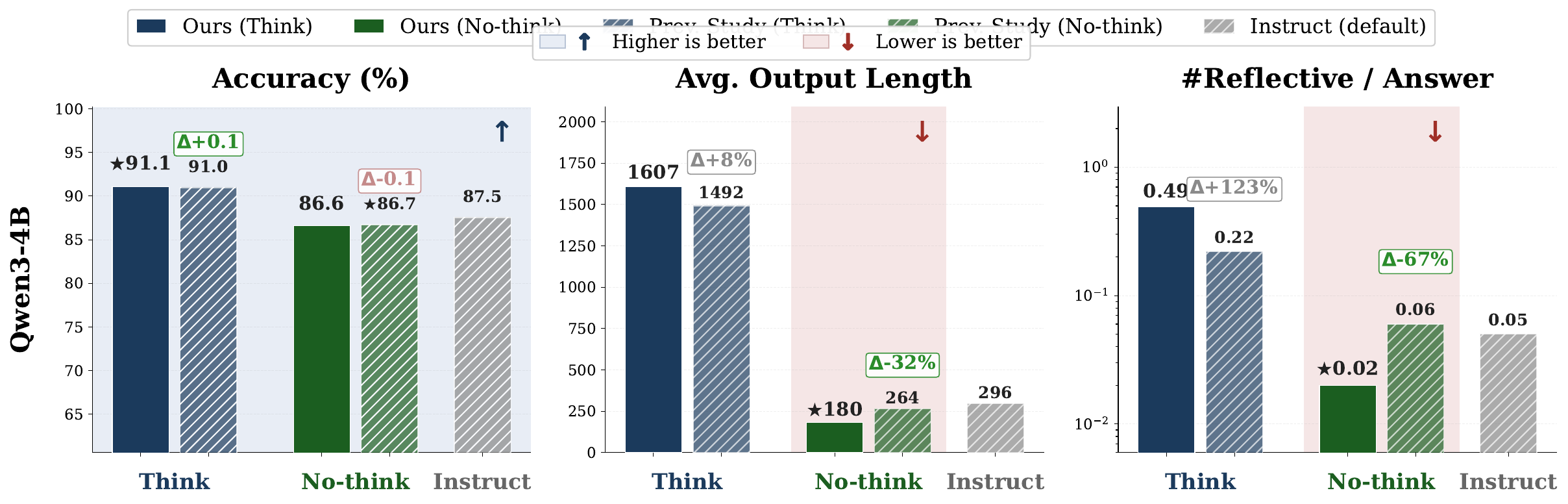}
    \caption{Accuracy, average output length, and per-answer reflective token count on MMLU-STEM for two \ple{} instantiations (Qwen2.5-7B, top; Qwen3-4B, bottom) compared with baselines, in the same format as \Cref{fig:exp_aime24}. ``Instruct'' is shown under its single default-mode response (no control-token suffix appended).}
    \label{fig:exp_composite_mmlu_stem}
\end{figure}

\begin{figure}[h!]
    \centering
    \includegraphics[width=1.0\linewidth]{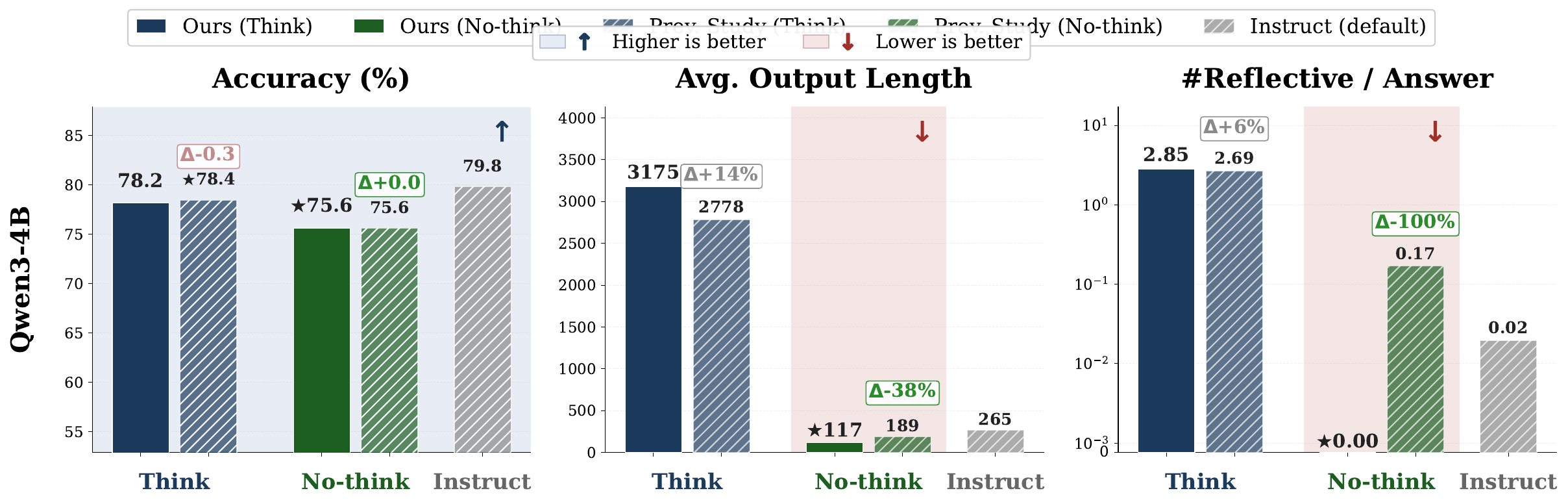}
    \caption{Accuracy, average output length, and per-answer reflective token count on full MMLU for the Qwen3-4B \ple{} instantiation compared with baselines, in the same format as \Cref{fig:exp_aime24}. ``Instruct'' is shown under its single default-mode response (no control-token suffix appended).}
    \label{fig:exp_composite_mmlu_full}
\end{figure}

\subsection{Model-Based Judge Validation of the \#Refl./Ans.\ Metric}
\label{sec:appendix_judge_validation}

To confirm that \#Refl./Ans.---a keyword-count metric---captures genuine implicit reasoning rather than surface token matches, we apply an external LLM judge (GPT-OSS-20B) to off-the-shelf hybrid-thinking baselines (\emph{not} \ple{}, to avoid circularity) across three model sizes (Qwen3-4B, Qwen3-8B, Qwen3-14B) and three benchmarks. The judge classifies each response three-way (no / implicit / explicit reasoning). \Cref{tab:appendix_judge_validation} reports agreement between the judge's verdict and \#Refl./Ans.: \nothink{} outputs flagged as leaking by \#Refl./Ans.\ are independently judged as containing implicit reasoning in 92--100\% of cases, and \think{} outputs with high \#Refl./Ans.\ are judged as explicit reasoning in 94--100\% of cases. \Cref{tab:appendix_judge_negative_control} further reports a true-negative control: a genuinely non-hybrid model (Qwen2.5-7B-Instruct) shows \#Refl.\ = 0 on every benchmark, and the judge independently and consistently returns ``no reasoning,'' confirming the metric does not fire on outputs that are already direct. Because this validation uses independent off-the-shelf baselines rather than \ple{} itself, it establishes that \#Refl./Ans.\ is a valid leakage signal in general, which then applies to our own results without circularity. 

\begin{table}[h!]
\centering
\caption{Model-based judge validation of \#Refl./Ans.\ as a genuine (not surface-level) leakage signal, on off-the-shelf Qwen3 hybrid baselines. Judge: GPT-OSS-20B, three-way no/implicit/explicit verdict.}
\label{tab:appendix_judge_validation}
\resizebox{\textwidth}{!}{
\begin{tabular}{llcccc}
\toprule
Benchmark & Model & Mode & Output Len. & \#Refl. & Judge verdict (agreement) \\
\midrule
\multirow{6}{*}{MATH500}
 & Qwen3-4B  & No-think & 969  & 210    & Implicit reasoning (100\%) \\
 & Qwen3-4B  & Think    & 4750 & 163160 & Explicit reasoning (98\%)  \\
 & Qwen3-8B  & No-think & 1044 & 113    & Implicit reasoning (92\%)  \\
 & Qwen3-8B  & Think    & 4419 & 154670 & Explicit reasoning (98\%)  \\
 & Qwen3-14B & No-think & 885  & 131    & Implicit reasoning (98\%)  \\
 & Qwen3-14B & Think    & 4468 & 127342 & Explicit reasoning (98\%)  \\
\midrule
\multirow{6}{*}{AIME24}
 & Qwen3-4B  & No-think & 5022  & 672   & Implicit reasoning (94\%)  \\
 & Qwen3-4B  & Think    & 11410 & 21481 & Explicit reasoning (94\%)  \\
 & Qwen3-8B  & No-think & 4677  & 24    & Implicit reasoning (98\%)  \\
 & Qwen3-8B  & Think    & 11578 & 23222 & Explicit reasoning (96\%)  \\
 & Qwen3-14B & No-think & 4013  & 74    & Implicit reasoning (98\%)  \\
 & Qwen3-14B & Think    & 11341 & 18814 & Explicit reasoning (98\%)  \\
\midrule
\multirow{6}{*}{GPQA-Diamond}
 & Qwen3-4B  & No-think & 1504 & 150    & Implicit reasoning (100\%) \\
 & Qwen3-4B  & Think    & 7615 & 132680 & Explicit reasoning (100\%) \\
 & Qwen3-8B  & No-think & 1245 & 196    & Implicit reasoning (96\%)  \\
 & Qwen3-8B  & Think    & 7481 & 131701 & Explicit reasoning (98\%)  \\
 & Qwen3-14B & No-think & 1294 & 181    & Implicit reasoning (98\%)  \\
 & Qwen3-14B & Think    & 7246 & 105493 & Explicit reasoning (96\%)  \\
\bottomrule
\end{tabular}
}
\end{table}

\begin{table}[h!]
\centering
\caption{True-negative control: a non-hybrid Instruct model shows zero reflective tokens and is consistently judged ``no reasoning,'' confirming the metric does not fire on genuinely direct outputs. Compare to the hybrid baseline (Qwen3-8B), which shows non-zero \#Refl.\ and is judged implicit reasoning on all three benchmarks.}
\label{tab:appendix_judge_negative_control}
\begin{tabular}{llccc}
\toprule
Benchmark & Model & Output Len. & \#Refl. & Judge verdict \\
\midrule
MATH500 & Qwen2.5-7B-Instruct & 703  & \textbf{0} & \textbf{No reasoning} \\
MATH500 & Qwen3-8B (hybrid)   & 1044 & 113        & Implicit reasoning \\
AIME24  & Qwen2.5-7B-Instruct & 1729 & \textbf{0} & \textbf{No reasoning} \\
AIME24  & Qwen3-8B (hybrid)   & 4677 & 24         & Implicit reasoning \\
GPQA-Diamond & Qwen2.5-7B-Instruct & 775 & \textbf{0} & \textbf{No reasoning} \\
GPQA-Diamond & Qwen3-8B (hybrid)   & 1245 & 196       & Implicit reasoning \\
\bottomrule
\end{tabular}
\end{table}

\subsection{Extended Related-Work Discussion: Positioning Against Learned-Router and Parameter-Adaptation Methods}
\label{sec:appendix_related_work_extended}

\ple{} is one instance of a broader family of \emph{mode-specific parameter separation} designs, which also includes adapter/LoRA-based approaches and learned or soft-routing Mixture-of-Experts (MoE) methods. We expand here on how \ple{} differs from each.

\textbf{Learned and soft-routing MoE.} Token-level MoE architectures~\citep{shazeer2017outrageously, fedus2022switch, lepikhin2021gshard, jiang2024mixtral, dai2024deepseekmoe} learn a gating network that selects experts per token, typically requiring auxiliary load-balancing losses and incurring routing overhead at every layer. Metis-HOME~\citep{lan2025metishome} is the closest such system to our setting: it applies a two-expert MoE with a \emph{learned} router to hybrid thinking, trained via a multi-stage RL+SFT pipeline. \ple{} differs on two axes simultaneously: (1) routing mechanism---\ple{} uses a deterministic, parameter-free control-token lookup resolved once per sequence, with no learnable gate, no routing supervision, and no balancing loss, whereas Metis-HOME's router is trained; and (2) modality and scope---Metis-HOME operates on cross-modal vision-language inputs, while \ple{} is text-only. These differences make a direct head-to-head comparison uninformative: substituting a learned router into our architecture would answer a different question (\emph{which router design is best?}) than the one this paper asks (\emph{is deterministic, parameter-free routing sufficient to achieve clean mode separation?}). We therefore position learned-router MoE methods, including Metis-HOME, as related work rather than as baselines.

\textbf{Adapters and LoRA-style methods.} An alternative approach to mode-specific behavior is to keep a single shared backbone and attach lightweight, mode-conditioned adapter or LoRA modules that are switched in or out per mode. Such methods trade parameter cost for a different failure mode: because the adapted capacity is small relative to the backbone, they may be less able to fully decouple two competing generation objectives (extended deliberation vs.\ concise direct answering) from the shared feed-forward pathway, which is precisely the interference \ple{} targets by giving each mode a full, independent MLP. We view adapter-based mode separation as a complementary, more parameter-efficient point in the same design space, and a natural direction for follow-up work rather than a claim this paper makes.

\textbf{General multi-task parameter separation.} Beyond hybrid thinking specifically, multi-task learning more broadly has explored task-specific sub-networks, mixture-of-adapters, and conditional computation to reduce cross-task interference. \ple{}'s contribution within this space is narrow and specific: a minimal, two-expert instantiation targeted at the think/no-think mode boundary, using the control token already present in the hybrid-thinking interface as a free, deterministic routing signal---requiring no additional supervision beyond what standard hybrid-thinking SFT already provides.

\Cref{tab:moe_comparison} summarizes these design-space differences along structural properties only; it makes no performance claims and reports no numbers, since we did not run these methods as baselines (see main text above for why).

\begin{table}[h!]
\centering
\caption{Structural comparison of mode-specific parameter-separation designs. Properties only---no performance numbers or claims of superiority, since these methods were not run as baselines in this paper (see \Cref{sec:related_work}).}
\label{tab:moe_comparison}
\footnotesize
\setlength{\tabcolsep}{4pt}
\renewcommand{\arraystretch}{1.15}

\textit{(a) Routing}\\[2pt]
\begin{tabular}{@{}>{\raggedright\arraybackslash}p{0.235\linewidth}>{\raggedright\arraybackslash}p{0.215\linewidth}>{\raggedright\arraybackslash}p{0.255\linewidth}>{\raggedright\arraybackslash}p{0.235\linewidth}@{}}
\toprule
Method & Routing signal & Learnable routing params & Granularity \\
\midrule
\textbf{\ple{} (ours)} & User-specified control token & None & Sequence-level (one decision per sequence) \\
Learned / soft-routing MoE & Learned gate (trained) & Yes & Token-level (per-token top-$k$) \\
Adapters / LoRA-style & Mode flag selects adapter & None for selection; adapter weights are trained & Sequence- or module-level \\
General multi-task separation & Task/mode identifier & Varies by method & Typically task-level \\
Metis-HOME~\citep{lan2025metishome} & Learned query/token-level gate & Yes & Token/query-level \\
\bottomrule
\end{tabular}

\vspace{8pt}
\textit{(b) Architecture and cost}\\[2pt]
\begin{tabular}{@{}>{\raggedright\arraybackslash}p{0.235\linewidth}>{\raggedright\arraybackslash}p{0.225\linewidth}>{\raggedright\arraybackslash}p{\dimexpr0.480\linewidth+8pt\relax}@{}}
\toprule
Method & Aux.\ balancing loss & What is duplicated \\
\midrule
\textbf{\ple{} (ours)} & None & Full MLP per mode \\
Learned / soft-routing MoE & Required & Full MLP per expert \\
Adapters / LoRA-style & None & Low-rank adapter, not full MLP \\
General multi-task separation & Varies by method & Task-specific sub-network (size varies) \\
Metis-HOME~\citep{lan2025metishome} & Required (MoE-style) & Full expert MLP (2 experts: think/non-think) \\
\bottomrule
\end{tabular}
\end{table}

\subsection{Routing Mechanism: Formal Description, Validation, and Worked Example}
\label{sec:appendix_routing_pseudocode}

Each input sequence is assigned a single routing index $r \in \{0, 1\}$ ($r{=}0$: \nothink{} $\to$ expert 0; $r{=}1$: \think{} $\to$ expert 1). This index is shared across all tokens and all layers for that sequence; there is no learnable gate or router parameter anywhere in the model.

\paragraph{Primary routing: registered control tokens.}
In the primary setup, \texttt{/think} and \texttt{/no\_think} are registered as dedicated special tokens, and routing is detected by an anchor-bounded, last-occurrence-wins scan over the prompt region. The detector locates the last \texttt{<|im\_start|>assistant} boundary token pair and searches only before it, structurally excluding the assistant's own response from influencing routing. Within the prompt, the last control token found is taken as the sequence's routing intent, which naturally handles multi-turn inputs containing multiple control words. If no control token is found, a safe-mode default applies: strict mode (training and strict evaluation) raises an error, while non-strict mode falls back to \texttt{config.default\_routing\_index}. The index is computed once at the start of \texttt{model.forward} and propagated to every layer; during generation, it is detected at prefill and cached for all subsequent decode steps, keeping prefill and decode routing consistent.

\paragraph{Fallback routing: unregistered control tokens.}
For backbones where the control words cannot be registered as atomic tokens, the same detector operates at the decoded-string level with guarded matching, so a single mechanism covers both regimes. Some tokenizers cannot register \texttt{/think} and \texttt{/no\_think} as atomic tokens; they split into subwords that can collide with ordinary text (e.g., the substring \texttt{\_th} inside \texttt{/think} colliding with a word like \texttt{R\_th}). The fallback detector operates on the \texttt{tokenizer.decode} string with a guarded regular expression: a negative lookbehind excludes the \texttt{</think>} closing tag, a negative lookahead blocks longer substrings such as \texttt{/think\_hard}, and the alternation order places \texttt{/no\_think} before \texttt{/think} to avoid a prefix-overlap match. Because \texttt{tokenizer.decode} maps both the registered-token and the multi-subword case to the same surface string, one detector implementation covers both regimes with no code branching.

\paragraph{Edge cases.} The detector explicitly handles: (i) no control token present in the prompt (safe-mode default, see above); (ii) multiple control tokens in a multi-turn prompt (last occurrence wins); (iii) control-token-like substrings appearing inside the model's own generated response (structurally excluded, since the scan region is bounded to end at the prompt/response anchor); and (iv) tokenizers that cannot register the control words as atomic tokens (string-level fallback with guarded regex, above).

\begin{table}[h!]
\centering
\caption{Routing detection robustness, validated on the full training set and a unit-test suite.}
\label{tab:appendix_routing_b1}
\begin{tabular}{@{}p{0.66\linewidth}c@{}}
\toprule
Property & Result \\
\midrule
Anchor (prompt/response boundary) found --- full training set & 37{,}500 / 37{,}500 (100\%) \\
Detection accuracy (decoded-string + guarded regex) & 37{,}500 / 37{,}500 (100\%) \\
Adversarial injection false-positives & 0 / 37{,}500 \\
Routing \& dispatch unit tests (detection / expert-selection correctness / backward-compat / batching) & 42 / 42 pass \\
\bottomrule
\end{tabular}
\end{table}

\paragraph{Worked example.} \Cref{tab:appendix_routing_trace} traces the routing decision end-to-end for a representative prompt ending in \texttt{/think} (``\dots leap year? \texttt{/think}''). The detector locates the last \texttt{<|im\_start|>assistant} anchor, scans only the prompt region to its left, finds \texttt{/think} as the last control-token match, and sets \texttt{routing\_index=1}, activating expert 1 (\think{}) for every layer and decoding step. Swapping only the final control token to \texttt{/no\_think} changes exactly the last three rows of the trace---the matched token, the routing index, and the activated expert---while every other step of the pipeline (attention, embeddings, normalization layers, the LM head) is identical.

\begin{table}[h!]
\centering
\caption{Typical-case token trace for a representative \texttt{/think}-routed prompt.}
\label{tab:appendix_routing_trace}
\begin{tabular}{ll}
\toprule
Stage & Content \\
\midrule
Prompt tail & \texttt{... leap year? /think <|im\_end|>} \\
Anchor (boundary) & \texttt{<|im\_start|>assistant} \\
Scan region & Everything left of the anchor (the prompt) \\
Decode $\to$ regex & Last match = \texttt{/think} \\
Decision & \texttt{routing\_index=1} $\to$ expert 1 (think), all layers, all steps \\
Output shape & Non-empty \texttt{<think>...</think>} + full answer \\
\bottomrule
\end{tabular}
\end{table}

\paragraph{Activated vs.\ total parameters.} \Cref{tab:appendix_routing_params} summarizes the parameter accounting underlying the efficiency claim in \Cref{sec:experiments}. The routing index selects one expert, and that expert's MLP is called directly for the whole sequence; there is no per-token gather/scatter and no dispatch machinery. One MLP is evaluated per sequence, so PLE's per-sequence activated parameters and MLP FLOPs equal a single dense model's by construction; the only overhead relative to a dense model is the storage of one extra expert MLP. 

\begin{table}[h!]
\centering
\caption{Total vs. activated parameters. PLE shares attention, embeddings, norms, and the LM head across both experts; only one additional expert MLP is stored, and sequence-level routing activates a single expert at inference.}
\label{tab:appendix_routing_params}
\begin{tabular}{lcc}
\toprule
 & Total params & Activated params / sequence \\
\midrule
Dense backbone & $P$ & $P$ \\
Two independent models & $2P$ & $P$ (per model) \\
\textbf{PLE} & $P + \Delta_{\text{MLP}}$ ($\approx 1.67P$) & $\mathbf{P}$ \\
\bottomrule
\end{tabular}
\end{table}

\subsection{Supplemental Results for Ablation Studies}
\label{sec:appendix_abl_results}
\subsubsection{Supplemental Results for Backbone Weight Ablation}
\label{sec:appendix_abl_results_base}

\Cref{tab:appendix_ablation_base_superior_mmlu_gpqa,tab:appendix_ablation_dataset_mmlu_gpqa} extend the ablations beyond the mathematical benchmarks and clarify the scope of our claims.

\paragraph{MMLU-STEM favors concise, nearly artifact-free direct answering.}
Across several initializations, \nothink{} accuracy on MMLU-STEM stays close to \think{} accuracy (typically within a few points, and exceeding it for some initializations) while remaining short and consistently near-artifact-free---reflective-token counts stay near zero in \nothink{} mode across every model we evaluate, including the off-the-shelf baseline. This may simply reflect that many MMLU-STEM questions can be answered correctly without extended reasoning. These results reinforce the practical motivation for hybrid thinking: a well-behaved \nothink{} mode should not be interpreted as a fallback setting.

\paragraph{GPQA-Diamond is more backbone-dependent.}
The GPQA results are more mixed. \nothink{} leakage remains low, but accuracy depends strongly on the initialization and training dataset. We view this as an important boundary of our claim. \ple{} addresses \emph{mode interference}; it does not by itself supply missing scientific knowledge or guarantee uniformly strong direct answering on every benchmark. The most stable conclusion on GPQA is therefore about control rather than universal accuracy gains.

\paragraph{Why these supplementary results matter.}
Taken together with the main math results, the appendix suggests a simple task-dependent picture. Architectural separation is most valuable when a model must clearly choose between ``reason deeply'' and ``answer directly.'' On benchmarks that mainly test retrieval or short reasoning, the main benefit of \ple{} is a clean and efficient \nothink{} mode; on long-horizon mathematical benchmarks, it can also substantially improve \nothink{} accuracy.

\subsection{Behavioral Probe of the Instruct Backbone}
\label{sec:appendix_instruct_probe}

\textit{Qwen3-4B-Instruct} exposes a single response mode, with no control token to select between
reasoning and direct answering. To characterise how that mode actually behaves, we evaluate the
off-the-shelf checkpoint twice: once on the bare prompt, and once with \texttt{/no\_think} appended.

\begin{table}[h!]
\centering
\caption{Behavioral probe of the off-the-shelf \textit{Qwen3-4B-Instruct} backbone. ``Plain'' is the bare
prompt; ``\texttt{/no\_think}'' appends the control token used by hybrid-thinking models.}
\vspace{-8pt}
\label{tab:appendix_instruct_probe}
\begin{tabular}{llccc}
\toprule
Benchmark & Prompt & Acc.\ (\%) & Avg.\ Len. & \#Refl./Ans. \\
\midrule
\multirow{2}{*}{AIME24}       & Plain              & 63.33 & 9932 & 2.06 \\
                              & \texttt{/no\_think} & 56.67 & 7950 & 1.81 \\
\midrule
\multirow{2}{*}{MATH500}      & Plain              & 94.52 & 1813 & 0.36 \\
                              & \texttt{/no\_think} & 94.36 & 1675 & 0.30 \\
\midrule
\multirow{2}{*}{GPQA-Diamond} & Plain              & 42.32 & 445  & 0.14 \\
                              & \texttt{/no\_think} & 42.02 & 433  & 0.10 \\
\midrule
\multirow{2}{*}{MMLU-STEM}    & Plain              & 87.54 & 296  & 0.05 \\
                              & \texttt{/no\_think} & 87.58 & 283  & 0.04 \\
\bottomrule
\end{tabular}
\end{table}

Two observations follow. First, the backbone reasons at length by default: on AIME24 it emits over
$9{,}000$ tokens carrying two reflective tokens per answer, so its single mode is a reasoning mode
rather than a direct-answer mode. Second, appending \texttt{/no\_think} does not give it a direct-answer
mode, however slightly but consistently reduces both its performance and reasoning leakage: it emits slightly less tokens and reflective tokens per answer than default mode on all 4 benchmarks. While on other three benchmarks its accuracy moves by at most $0.3$ points, on AIME24, the accuracy drops by $6.66\%$,  Both observations are
consistent with its lineage as a non-thinking release derived from a hybrid-thinking
base~\citep{yang2025qwen3,qwen2507card2026}, and they motivate supplying the direct-answer path
architecturally rather than through prompting.

\subsection{Two-Separate-Models Controlled Ablation}
\label{sec:appendix_specialists}
All models below are trained from an identical Qwen3-4B-Base checkpoint on matched data, isolating the effect of deterministic routing from any difference in base weights or data. ``Specialist'' models are trained on only one mode's data; ``weight-merge'' naively averages the two specialists' weights with no routing mechanism.

\begin{table}[h!]
\centering
\caption{\textit{Base-PLE} (\ple{} initialized from Qwen3-4B-Base) vs.\ single-mode specialists and a naive weight-merge, all trained from an identical Qwen3-4B-Base checkpoint on matched data. Accuracy (\%) shown; per-benchmark tables with output length and \#Refl./Ans.\ follow below.}
\label{tab:specialists_summary}
\begin{tabular}{lccc}
\toprule
Model & MATH500 & AIME24 & GPQA-Diamond \\
\midrule
Think-only specialist (think mode)   & 80.50 & 22.00 & 30.30 \\
No\_think-only specialist (no\_think mode) & 64.30 & 4.70 & 31.00 \\
Weight-merge of the two specialists   & 66.50 & 9.30 & 32.90 \\
\midrule
\textbf{\textit{Base-PLE}, think}         & \textbf{82.40} & 21.33 & \textbf{34.00} \\
\textbf{\textit{Base-PLE}, no\_think}       & \textbf{64.50} & 3.33 & \textbf{34.00} \\
\bottomrule
\end{tabular}
\end{table}

\begin{table}[h!]
\centering
\caption{AIME24 --- \ple{} vs.\ single-mode specialists and weight-merge.}
\label{tab:appendix_specialists_aime24}
\begin{tabular}{lccc}
\toprule
Model & Acc.\ (\%) & Avg.\ Len. & \#Refl./Ans. \\
\midrule
Qwen3-4B-Base + think-only data    & 22.00 & 28895 & 13.53 \\
Qwen3-4B-Base + no\_think-only data & 4.70  & 941   & 0.07  \\
Weight-merge of the two specialists & 9.30  & 4062  & 3.76  \\
\midrule
\textbf{\textit{Base-PLE}, think}    & \textbf{21.33} & 27698 & 12.37 \\
\textbf{\textit{Base-PLE}, no\_think} & \textbf{3.33}  & 1244  & 0.09  \\
\bottomrule
\end{tabular}
\end{table}

\begin{table}[h!]
\centering
\caption{MATH500 --- \ple{} vs.\ single-mode specialists and weight-merge.}
\label{tab:appendix_specialists_math500}
\begin{tabular}{lccc}
\toprule
Model & Acc.\ (\%) & Avg.\ Len. & \#Refl./Ans. \\
\midrule
Qwen3-4B-Base + think-only data    & 80.50 & 10390 & 4.73 \\
Qwen3-4B-Base + no\_think-only data & 64.30 & 358   & 0.02 \\
Weight-merge of the two specialists & 66.50 & 1044  & 2.12 \\
\midrule
\textbf{\textit{Base-PLE}, think}    & \textbf{82.40} & 8514 & 4.64 \\
\textbf{\textit{Base-PLE}, no\_think} & \textbf{64.50} & 368  & 0.02 \\
\bottomrule
\end{tabular}
\end{table}

\begin{table}[h!]
\centering
\caption{GPQA-Diamond --- \ple{} vs.\ single-mode specialists and weight-merge.}
\label{tab:appendix_specialists_gpqa}
\begin{tabular}{lccc}
\toprule
Model & Acc.\ (\%) & Avg.\ Len. & \#Refl./Ans. \\
\midrule
Qwen3-4B-Base + think-only data    & 30.30 & 19742 & 24.86 \\
Qwen3-4B-Base + no\_think-only data & 31.00 & 457   & 0.21  \\
Weight-merge of the two specialists & 32.90 & 1152  & 2.95  \\
\midrule
\textbf{\textit{Base-PLE}, think}    & \textbf{34.00} & 18433 & 37.33 \\
\textbf{\textit{Base-PLE}, no\_think} & \textbf{34.00} & 397   & 0.02  \\
\bottomrule
\end{tabular}
\end{table}

The initialization-equivalence test (12/12 prompt--mode pairs byte-identical; 108/108 parameter tensors numerically equal at init) and the parameter accounting ($P + \Delta_{\text{MLP}} \approx 1.67P$ for \ple{} vs.\ $2P$ for two independent models, a $\approx 0.33P$ saving) are reported in Appendix~\ref{sec:appendix_routing_pseudocode} (\Cref{tab:appendix_routing_params}); we do not repeat the derivation here to avoid duplication.

\subsection{Generalization: Llama-3.1-8B and Full MMLU}
\label{sec:appendix_generalization}
To test generalization beyond the Qwen family and beyond math/STEM tasks, we apply \ple{} to \textit{Llama-3.1-8B} (trained on OpenR1 data at the 20k scale, matching the data scale of the SFT-only comparison recipe) and evaluate on full MMLU (57 subjects, 14{,}042 questions) with the Qwen3-4B backbone used throughout the main paper. In \Cref{tab:appendix_llama_math500,tab:appendix_llama_aime24,tab:appendix_llama_gpqa,tab:appendix_llama_mmlustem}, the ``Instruct'' row is the off-the-shelf Llama-3.1-8B-Instruct model (no-think only, no hybrid capability); ``+SFT only'' is a dense two-phase SFT hybrid recipe following \citet{wang2025demystifying} at the matched 20k data scale; ``PLE'' is our architecture applied to the same backbone and data.

\begin{table}[h!]
\centering
\caption{Cross-family generalization (Llama-3.1-8B): MATH500.}
\label{tab:appendix_llama_math500}
\begin{tabular}{lcc}
\toprule
Model & Mode & Acc.\ (\%) \\
\midrule
Llama-3.1-8B-Instruct          & No-think & 49.3 \\
Llama-3.1-8B-Instruct + SFT only & No-think & 47.2 \\
Llama-3.1-8B-Instruct + SFT only & Think    & 68.6 \\
\textbf{Llama-3.1-8B-PLE}       & No-think & \textbf{56.0} \\
\textbf{Llama-3.1-8B-PLE}       & Think    & 66.4 \\
\bottomrule
\end{tabular}
\end{table}

\begin{table}[h!]
\centering
\caption{Cross-family generalization (Llama-3.1-8B): AIME24.}
\label{tab:appendix_llama_aime24}
\begin{tabular}{lcc}
\toprule
Model & Mode & Acc.\ (\%) \\
\midrule
Llama-3.1-8B-Instruct          & No-think & 6.0 \\
Llama-3.1-8B-Instruct + SFT only & No-think & 4.3 \\
Llama-3.1-8B-Instruct + SFT only & Think    & 7.0 \\
\textbf{Llama-3.1-8B-PLE}       & No-think & \textbf{12.0} \\
\textbf{Llama-3.1-8B-PLE}       & Think    & \textbf{18.0} \\
\bottomrule
\end{tabular}
\end{table}

\begin{table}[h!]
\centering
\caption{Cross-family generalization (Llama-3.1-8B): GPQA-Diamond.}
\label{tab:appendix_llama_gpqa}
\begin{tabular}{lcc}
\toprule
Model & Mode & Acc.\ (\%) \\
\midrule
Llama-3.1-8B-Instruct          & No-think & 30.5 \\
Llama-3.1-8B-Instruct + SFT only & No-think & 26.7 \\
Llama-3.1-8B-Instruct + SFT only & Think    & 36.3 \\
\textbf{Llama-3.1-8B-PLE}       & No-think & \textbf{32.0} \\
\textbf{Llama-3.1-8B-PLE}       & Think    & \textbf{38.0} \\
\bottomrule
\end{tabular}
\end{table}

\begin{table}[h!]
\centering
\caption{Cross-family generalization (Llama-3.1-8B): MMLU-STEM.}
\label{tab:appendix_llama_mmlustem}
\begin{tabular}{lcc}
\toprule
Model & Mode & Acc.\ (\%) \\
\midrule
Llama-3.1-8B-Instruct          & No-think & 53.0 \\
Llama-3.1-8B-Instruct + SFT only & No-think & 53.4 \\
Llama-3.1-8B-Instruct + SFT only & Think    & 77.3 \\
\textbf{Llama-3.1-8B-PLE}       & No-think & \textbf{76.7} \\
\textbf{Llama-3.1-8B-PLE}       & Think    & \textbf{78.9} \\
\bottomrule
\end{tabular}
\end{table}

\begin{table}[h!]
\centering
\caption{Full MMLU (57 subjects, 14{,}042 questions), Qwen3-4B backbone; the same results are shown graphically in \Cref{fig:exp_composite_mmlu_full}. \textdagger: PLE think accuracy is over the full question set; average length and \#Refl./response for PLE think are over a 90.6\% subset (56/57 subjects; \texttt{professional\_law} partially truncated).}
\label{tab:appendix_mmlu_full}
\begin{tabular}{lcccc}
\toprule
Model & Mode & Acc.\ (\%) & Avg.\ Len. & \#Refl./response \\
\midrule
No\_think-only SFT (single mode) & No-think & 46.2 & 268  & 0.001 \\
Think-only SFT (single mode)     & Think    & 75.2 & 3723 & 36.804 \\
Qwen3-4B (off-the-shelf hybrid)  & No-think & 75.1 & 583  & 0.003 \\
Qwen3-4B (off-the-shelf hybrid)  & Think    & 79.8 & 1916 & 13.500 \\
SFT-recipe~\citep{wang2025demystifying}         & No-think & 57.1 & 1548 & 0.646 \\
SFT-recipe~\citep{wang2025demystifying}         & Think    & 76.9 & 2823 & 19.548 \\
\midrule
\textbf{\ple{} (ours)} & No-think & \textbf{73.4} & 164 & 0.003 \\
\textbf{\ple{} (ours)} & Think    & 77.3\textdagger & 2349\textdagger & 4.803\textdagger \\
\bottomrule
\end{tabular}
\end{table}

\subsection{Ruling Out Data-Construction Bias}
\label{sec:appendix_data_bias}
Because \ple{} and its dense SFT baseline (the ``SFT-only'' row in \Cref{tab:exp_leakage}) are trained on the identical Qwen3-235B-distilled dataset---same teacher, same filtering pipeline, same examples---any teacher-style bias in the \nothink{} targets would inflate both equally, and cannot selectively favor \ple{}. As a direct test, we further retrain \ple{} on a newer, more aggressively \emph{condensed} \nothink{} dataset: a stronger model rewrites each full chain-of-thought into a concise, reasoning-free answer, an even more aggressive direct-answer construction than the original filtering pipeline. If teacher-style distillation were inflating \nothink{} performance, this more aggressively condensed data should help further; instead, \Cref{tab:appendix_data_bias} shows it makes \nothink{} accuracy substantially \emph{worse} (AIME24: 40.0\%~$\to$~3.3\%; MATH500: 80.8\%~$\to$~64.3\%), with \think{} accuracy also somewhat lower. This is the opposite of what a teacher-bias explanation predicts, and is consistent with \ple{}'s benefit being architectural rather than a data-construction artifact. We note that identifying the single best \nothink{}-construction recipe is outside this paper's scope; the architecture is a separable, independently improvable component.

\begin{table}[h!]
\centering
\caption{Data-construction-bias ablation: \ple{} retrained on the original vs.\ a more aggressively condensed \nothink{} dataset (Qwen3-4B-Base). Condensation worsens \nothink{} accuracy, ruling out teacher-style bias inflation as an explanation for \ple{}'s gains.}
\label{tab:appendix_data_bias}
\begin{tabular}{llcc}
\toprule
Benchmark & Training data & Mode & Acc.\ (\%) \\
\midrule
\multirow{4}{*}{AIME24}
 & Original data (as in the paper) & Think    & 60.0 \\
 & Original data (as in the paper) & No-think & 40.0 \\
 & Condensed \nothink{} data (new)  & Think    & 53.3 \\
 & Condensed \nothink{} data (new)  & No-think & \textbf{3.3} \\
\midrule
\multirow{4}{*}{MATH500}
 & Original data (as in the paper) & Think    & 94.8 \\
 & Original data (as in the paper) & No-think & 80.8 \\
 & Condensed \nothink{} data (new)  & Think    & 91.0 \\
 & Condensed \nothink{} data (new)  & No-think & \textbf{64.3} \\
\bottomrule
\end{tabular}
\end{table}

\begin{figure}[h!]
\centering
\includegraphics[width=\linewidth]{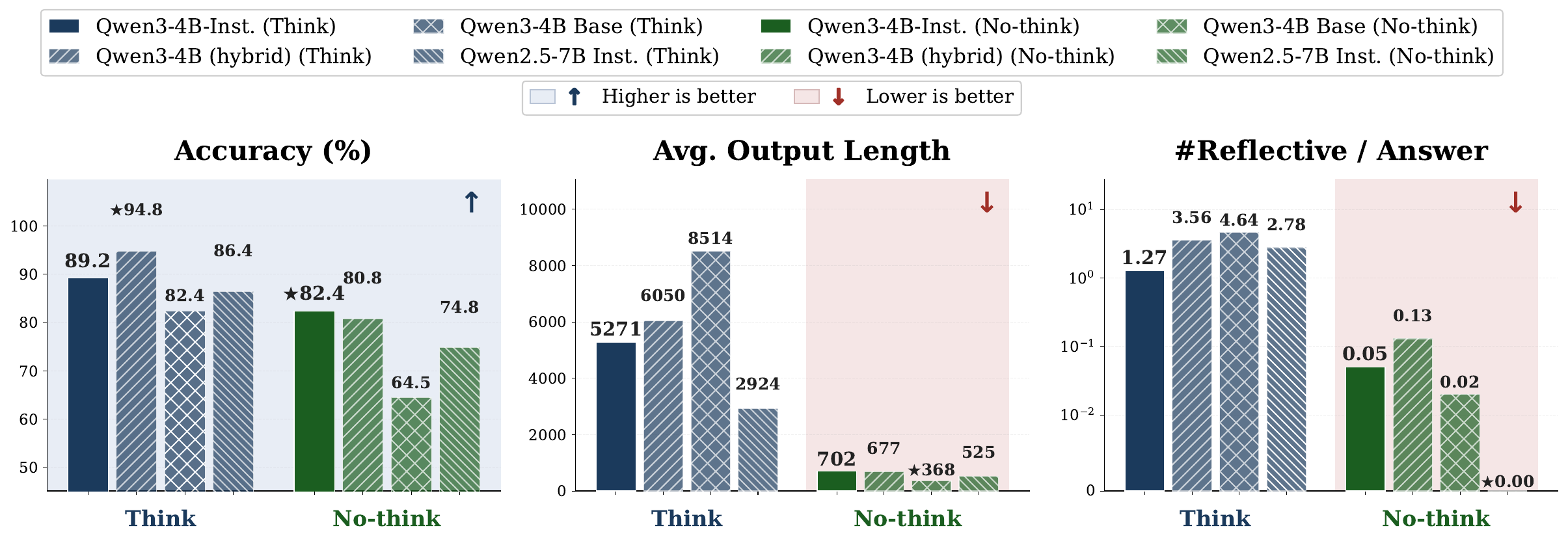}
\caption{Ablation on backbone choice (MATH500); same models as \Cref{fig:ablation_base_aime24}. PLE is initialized from ``Qwen2.5-7B-Instruct'' (pure instruction-tuned weights), ``Qwen/Qwen3-4B-Instruct-2507'', ``Qwen3-4B Base'' (raw pretrained weights), or ``Qwen3-4B (hybrid)'' (hybrid-thinking weights). Residual \nothink{} leakage again orders with each backbone's own reasoning behavior (0.00 for Qwen2.5-7B-Instruct\ $<$ 0.05 for Qwen/Qwen3-4B-Instruct-2507\ $<$ 0.13 for the hybrid variant). Unlike AIME24, the pretrained base stays usable on this easier benchmark, though at the lowest \nothink{} accuracy.}
\label{fig:ablation_base_math500}
\end{figure}

\begin{table*}[h!]
    \centering
    \caption{Backbone-weight ablation under \textbf{superior-reasoning} training data on MMLU-STEM and GPQA-Diamond. We report accuracy (Acc., \%), output length (Len.), and reflective tokens per answer (\#Refl./Ans.).}
    \vspace{-8pt}
    \label{tab:appendix_ablation_base_superior_mmlu_gpqa}
    \setlength{\tabcolsep}{5pt}
    \begin{tabular}{cccccccc}
    \toprule
    \multirow{2}{*}{Model} & \multirow{2}{*}{Mode}
    & \multicolumn{3}{c}{MMLU-STEM} & \multicolumn{3}{c}{GPQA-Diamond} \\
    \cmidrule(lr){3-5} \cmidrule(lr){6-8}
    & & Acc. & Len. & \#Refl./Ans. & Acc. & Len. & \#Refl./Ans. \\
    \midrule
    \multirow{2}{*}{\shortstack{Qwen3-4B\\Base}}
     & Think    & 94.44 & \-- & \-- & 30.00 & \-- & \-- \\
     & No-think & 88.89 & \-- & \-- & 44.00 & \-- & \-- \\
    \midrule
    \multirow{2}{*}{Qwen3-4B}
     & Think    & 91.11 & 794 & 6.36 & 42.00 & 8268 & 3.26 \\
     & No-think & 93.33 & 117 & 0.06 & 26.00 & 298  & 0.00 \\
    \midrule
    \multirow{2}{*}{\shortstack{Qwen3-4B\\Instruct}}
     & Think    & 91.06 & 1606.65 & 0.49 & 49.49 & 15013.74 & 40.43 \\
     & No-think & 86.61 & 180.19  & \textbf{0.02} & 45.86 & 736.34   & \textbf{0.23} \\
    \midrule
    \multirow{2}{*}{\shortstack{Qwen2.5-7B\\Instruct}}
     & Think    & 88.89 & 429 & 8.16 & 34.00 & 7044 & 4.44 \\
     & No-think & 92.22 & 139 & 0.02 & 30.00 & 3857 & 0.00 \\
    \bottomrule
    \end{tabular}
\end{table*}

\subsection{Supplemental Results for Dataset Ablation}
\label{sec:appendix_abl_results_dataset}
\begin{figure}[h!]
    \centering
    \includegraphics[width=1.0\linewidth]{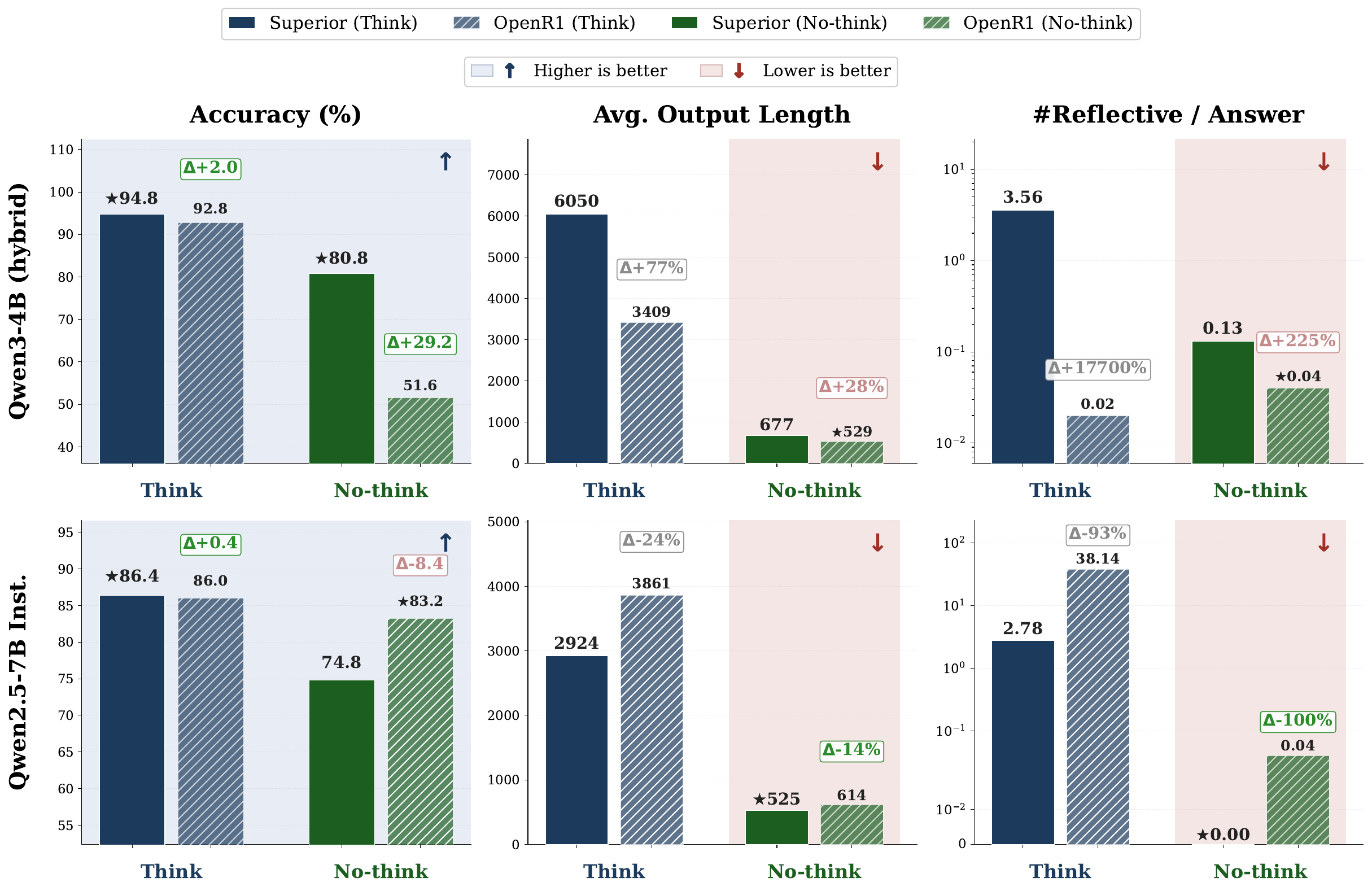}
    \caption{Ablation on training dataset (MATH500) with the native hybrid-thinking \textit{Qwen3-4B} backbone (top) and \textit{Qwen2.5-7B-Instruct} (bottom). ``Superior'': PLE trained on our high-difficulty dataset with longer CoT traces; ``OpenR1'': trained on simpler data. On the stronger Qwen3-4B backbone Superior wins in both modes, while on Qwen2.5-7B OpenR1 remains competitive---the same capacity-dependent pattern seen on AIME24 (\Cref{fig:ablation_dataset_aime24}).}
    \label{fig:ablation_dataset_math500}
\end{figure}

\begin{table*}[h!]
    \centering
    \caption{Dataset ablation across different backbone weights.
    Leftmost column groups results by dataset (\textbf{OpenR1} vs. \textbf{superior-reasoning 54k}); within each group, we report think/no-think performance for each backbone on MATH500 and AIME24. The \textit{Qwen3-4B} rows use the native hybrid-thinking checkpoint.}
    \vspace{-8pt}
    \label{tab:ablation_dataset_by_base_weight}
    \setlength{\tabcolsep}{4pt}
    \resizebox{\textwidth}{!}{
   \begin{tabular}{ccccccccc}
    \toprule
    \multirow{2}{*}{Dataset} & \multirow{2}{*}{Backbone} & \multirow{2}{*}{Mode}
    & \multicolumn{3}{c}{MATH500} & \multicolumn{3}{c}{AIME24} \\
    \cmidrule(lr){4-6} \cmidrule(lr){7-9}
    & & & Acc. & Len. & \#Refl./Ans. & Acc. & Len. & \#Refl./Ans. \\
    \midrule
    \multirow{6}{*}{OpenR1}
    & \multirow{2}{*}{\shortstack{Qwen3-4B\\Base}}
      & Think    & 80.00 & 16384 & 28.80 & 50.00 & 16384 & 19.40 \\
    & & No-think & 56.00 & 16262 & 0.00  & 10.00 & 16384 & 0.00 \\
    \cmidrule(lr){2-9}
    & \multirow{2}{*}{\shortstack{Qwen3-4B\\(hybrid)}}
      & Think    & 92.80 & 3409  & 0.02  & 52.00 & 9775  & 25.98 \\
    & & No-think & 51.60 & 529   & 0.04  & 16.00 & 906   & 0.00 \\
    \cmidrule(lr){2-9}
    & \multirow{2}{*}{\shortstack{Qwen2.5-7B\\Instruct}}
      & Think    & 86.00 & 3861  & 38.14 & 36.00 & 11794 & 62.76 \\
    & & No-think & 83.20 & 614   & 0.04  & 24.00 & 1329  & 0.02 \\
    \midrule
    \multirow{6}{*}{\shortstack{Superior-\\reasoning\\54k}}
    & \multirow{2}{*}{\shortstack{Qwen3-4B\\Base}}
      & Think    & 90.00 & 15769 & 7.00 & 20.00 & 16384 & 5.00 \\
    & & No-think & 74.00 & 16224 & 0.00 & 30.00 & 16384 & 0.00 \\
    \cmidrule(lr){2-9}
    & \multirow{2}{*}{\shortstack{Qwen3-4B\\(hybrid)}}
      & Think    & 94.80 & 6050 & 3.56 & 60.00 & 31733 & 7.02 \\
    & & No-think & 80.80 & 677  & 0.13 & 40.00 & 4708  & 0.39 \\
    \cmidrule(lr){2-9}
    & \multirow{2}{*}{\shortstack{Qwen2.5-7B\\Instruct}}
      & Think    & 86.40 & 2924 & 2.78 & 32.00 & 13809 & 8.16 \\
    & & No-think & 74.80 & 525  & 0.00 & 22.00 & 4120  & 0.02 \\
    \bottomrule
    \end{tabular}
    }
\end{table*}

\begin{table*}[h!]
    \centering
    \caption{Dataset ablation across different backbone weights on MMLU-STEM and GPQA-Diamond.
    Leftmost column groups results by dataset (\textbf{OpenR1} vs. \textbf{superior-reasoning}); within each group, we report think/no-think performance for each backbone. The \textit{Qwen3-4B} rows use the native hybrid-thinking checkpoint.}
    \vspace{-8pt}
    \label{tab:appendix_ablation_dataset_mmlu_gpqa}
    \setlength{\tabcolsep}{4pt}
    \resizebox{\textwidth}{!}{
    \begin{tabular}{ccccccccc}
    \toprule
    \multirow{2}{*}{Dataset} & \multirow{2}{*}{Backbone} & \multirow{2}{*}{Mode}
    & \multicolumn{3}{c}{MMLU-STEM} & \multicolumn{3}{c}{GPQA-Diamond} \\
    \cmidrule(lr){4-6} \cmidrule(lr){7-9}
    & & & Acc. & Len. & \#Refl./Ans. & Acc. & Len. & \#Refl./Ans. \\
    \midrule
    \multirow{8}{*}{OpenR1}
    & \multirow{2}{*}{\shortstack{Qwen3-4B\\Base}}
      & Think    & 33.33 & 16384 & 19.40 & 8.00  & 16384 & 14.10 \\
    & & No-think & 83.33 & 16384 & 0.00  & 24.00 & 16384 & 0.00 \\
    \cmidrule(lr){2-9}
    & \multirow{2}{*}{\shortstack{Qwen3-4B\\(hybrid)}}
      & Think    & 88.89 & 999 & 25.98 & 66.00 & 5372 & 16.88 \\
    & & No-think & 83.33 & 160 & 0.00  & 30.00 & 603  & 0.00 \\
    \cmidrule(lr){2-9}
    & \multirow{2}{*}{\shortstack{Qwen2.5-7B\\Instruct}}
      & Think    & 77.78 & 848 & 62.76 & 6.00  & 1821 & 41.00 \\
    & & No-think & 88.89 & 352 & 0.02  & 40.00 & 871  & 0.00 \\
    \midrule
    \multirow{8}{*}{\shortstack{Superior-\\reasoning}}
    & \multirow{2}{*}{\shortstack{Qwen3-4B\\Base}}
      & Think    & 94.44 & \-- & \-- & 30.00 & \-- & \-- \\
    & & No-think & 88.89 & \-- & \-- & 44.00 & \-- & \-- \\
    \cmidrule(lr){2-9}
    & \multirow{2}{*}{\shortstack{Qwen3-4B\\(hybrid)}}
      & Think    & 91.11 & 794 & 6.36 & 42.00 & 8268 & 3.26 \\
    & & No-think & 93.33 & 117 & 0.06 & 26.00 & 298  & 0.00 \\
    \cmidrule(lr){2-9}
    & \multirow{2}{*}{\shortstack{Qwen2.5-7B\\Instruct}}
      & Think    & 88.89 & 429 & 8.16 & 34.00 & 7044 & 4.44 \\
    & & No-think & 92.22 & 139 & 0.02 & 30.00 & 3857 & 0.00 \\
    \bottomrule
    \end{tabular}
    }
\end{table*}

\subsection{Broader Impact and Future Applications}
\label{sec:appendix_broader_impact}

\paragraph{Agents.}
Agent trajectories interleave routine steps---parsing a tool result, filling a template, deciding which
file to open---with a much smaller number of genuinely deliberative ones~\citep{yao2023react}. Because agent
benchmarks score long horizons in which each step conditions the next, residual leakage compounds rather than
staying local, so a \nothink{} mode that can be trusted lets a system spend tokens only where deliberation is
actually needed~\citep{liu2024agentbench,yang2026agentce}. The same property helps downstream: records
written back into long-term memory stay compact when the underlying responses are direct rather than padded
with abandoned reasoning~\citep{shinn2023reflexion,long2026memtrace,chen2026agenticmemory}. It also composes with how agent
capabilities are increasingly packaged---a skill library can declare the reasoning mode a skill expects at
call time~\citep{wang2024voyager,yang2026agentskills}, and persistent multi-agent workspaces, where many
calls are routine bookkeeping, benefit from the same effect~\citep{wang2026atwz}.

\paragraph{Conventional inference serving.}
In interactive deployment most everyday queries want a short, direct answer, and a \nothink{} mode that
reliably delivers one lowers both the cost of a request and the time until the user has a complete response,
while \think{} remains available for the queries that need it. This acts on a different axis from most
inference-time work: it reduces how many tokens are generated in the first place rather than the cost of
producing each one. IO-aware exact attention and lossless weight compression leave the token stream
untouched~\citep{dao2022flashattention,zhang2025dfloat11}, as does post-training weight
quantization or low-rank compression~\citep{frantar2023optq,lin2024awq,wang2025proteinlm,li2026jwsvd}, though quantization can also shift model behaviour rather than merely compress it~\citep{fu2025quantized}. Cache-side methods instead target the memory
that accumulates as generation proceeds~\citep{liu2024kivi,hariri2026quantize,li2025faedkv}. Because these axes are
largely independent, an architectural mode boundary can be stacked with any of them.

\paragraph{Model architecture.}
Conditional computation is the established route to spending capacity selectively, activating a subset
of parameters per token through learned sparse experts~\citep{fedus2022switch,jiang2024mixtral}. \ple{}
occupies a deliberately minimal point in that space: two experts, no router parameters, and a routing
decision fixed once per sequence by a token the interface already exposes. Methods that instead shorten
reasoning after the decision to reason has been made~\citep{sui2025stop} are complementary, since they act on
the length of a \think{} trace rather than on the boundary between modes. Both sit inside the broader question of how inference-time compute should be allocated and compared across regimes~\citep{hariri2026testtime}.

\paragraph{Post-training.}
The \nothink{} behaviour \ple{} produces is a property of post-training, not of the architecture alone.
Post-training effects are known to interact across domains in order-sensitive ways, and training-level
mitigation of leakage is itself sensitive to data scale and
mixture~\citep{yang2026domains,wang2025demystifying}. Mode separation should therefore be re-measured
whenever the training mixture changes, rather than assumed to transfer from one recipe to another.

\clearpage

\newpage

\section{Formal Architecture Specification}\label{sec:formal_architecture_specification}
\textbf{Path-locked dual-expert decoder} Consider a standard pre-norm decoder layer:
\begin{align}
    \mathbf{h}^{(l)}_{\text{attn}}
    &= \mathbf{h}^{(l)} + \mathrm{Attn}^{(l)}\!\left(\mathrm{LN}^{(l)}_1(\mathbf{h}^{(l)})\right); 
    \mathbf{h}^{(l+1)}
    = \mathbf{h}^{(l)}_{\text{attn}} + \mathrm{MLP}^{(l)}\!\left(\mathrm{LN}^{(l)}_2(\mathbf{h}^{(l)}_{\text{attn}})\right).
\end{align}

\ple{} changes only the second part. Each decoder MLP is replaced by a pair of experts, and the selected route \(r \in \{0,1\}\) determines which expert is active:
\begin{align}
    \mathbf{h}^{(l+1)}
    = \mathbf{h}^{(l)}_{\text{attn}} + \mathrm{MLP}^{(l)}_{r}\!\left(\mathrm{LN}^{(l)}_2(\mathbf{h}^{(l)}_{\text{attn}})\right),
\end{align}
where \(r=0\) denotes the \nothink{} expert and \(r=1\) denotes the \think{} expert. For the Qwen backbones used in our main experiments, each expert preserves the source model's original SwiGLU-style MLP parameterization:
\begin{align}
    \mathrm{MLP}^{(l)}_{r}(\mathbf{x})
    =
    W^{(l,r)}_{\mathrm{down}}
    \Big(
        \phi\!\left(W^{(l,r)}_{\mathrm{gate}}\mathbf{x}\right)
        \odot
        W^{(l,r)}_{\mathrm{up}}\mathbf{x}
    \Big),
\end{align}
where \(\phi(\cdot)\) is the activation function and \(\odot\) is elementwise multiplication. The two experts are structurally identical; they differ only in parameters and in which control mode activates them. All non-MLP components remain shared across modes, including self-attention, positional encoding, normalization, token embeddings, and the LM head. We duplicate only the MLP because the two modes should share the same contextualization and stored knowledge, but need not share the final feed-forward transformation that turns those representations into long-form reasoning or concise direct answers. This makes \ple{} a targeted intervention rather than a full model duplication scheme.
\textbf{Deterministic control-token routing} \ple{} does not use a learned router. Instead, routing is determined by control tokens in the input. We register two tokenizer-level special tokens, \texttt{/no\_think} and \texttt{/think}, each represented by a single token ID. Given an input sequence \(\mathbf{x}\), the routing function is
\begin{align}
r(\mathbf{x}) =
\begin{cases}
1, & \text{if the final control token in } \mathbf{x} \text{ is } \texttt{/think},\\
0, & \text{if the final control token in } \mathbf{x} \text{ is } \texttt{/no\_think},\\
r_{\mathrm{default}}, & \text{if no control token is present},
\end{cases}
\end{align}
where \(r_{\mathrm{default}}=0\) routes to the \nothink{} expert. The \emph{last control token wins} rule serves two purposes. First, it makes routing fully deterministic and interpretable from the prompt alone. Second, it allows a later template or system token to override an earlier user-supplied token without introducing any learned gating behavior.
During autoregressive generation, the route is determined from the accumulated prompt at the first decoding step and then cached and reused for all later steps. This ensures that the active expert remains fixed throughout decoding. Our current implementation uses one scalar routing index per forward pass. Accordingly, we do not mix \think{} and \nothink{} examples within the same batch; each forward pass uses a single mode. This keeps the routing logic simple while preserving the intended turn-level control semantics.

\textbf{Routing-conditioned supervised fine-tuning} We fine-tune \ple{} with standard causal language modeling loss on chat-formatted examples tagged with either \texttt{/think} or \texttt{/no\_think}. Let \(\mathcal{B}_r\) denote a minibatch routed to mode \(r \in \{0,1\}\). The training objective is
\begin{align}
\mathcal{L}(\theta;\mathcal{B}_r)
=
-
\sum_{(\mathbf{x},\mathbf{y}) \in \mathcal{B}_r}
\sum_{t=1}^{|\mathbf{y}|}
\log p_{\theta}\!\left(y_t \mid \mathbf{x}, \mathbf{y}_{<t}; r\right).
\end{align}

Let \(\theta_{\mathrm{sh}}\) denote the shared parameters and \(\theta_0,\theta_1\) denote the two expert parameter sets. Because only the routed expert is executed, for any batch using route \(r\),
\begin{align}
    \frac{\partial \mathcal{L}}{\partial \theta_{1-r}} = 0,
    \qquad
    \frac{\partial \mathcal{L}}{\partial \theta_r} \neq 0,
    \qquad
    \frac{\partial \mathcal{L}}{\partial \theta_{\mathrm{sh}}} \neq 0.
\end{align}

The inactive expert therefore receives no gradient, while the shared backbone is updated by both modes. This is the central mechanism of \ple{}: it removes direct parameter sharing between the two MLP pathways while still allowing both modes to benefit from a common representational substrate. At the same time, because attention and output layers remain shared, \ple{} should be understood as \emph{targeted architectural separation}, not as full model duplication.

\section{Theoretical account: Path-Lock as partial block-diagonalization}
\label{sec:theory}

This section formalizes \ple{} as a partially decoupled two-objective optimization problem.
The derivation is guided by three observations from prior work:
(i) multi-task learning is inherently multi-objective and can suffer from conflicting gradients
\citep{sener2018multi,yu2020gradient,liu2021cagrad};
(ii) expert routing is useful when different objectives prefer different submodels
\citep{jordan1994hme,ma2018mmoe};
and (iii) Transformer feed-forward layers are a privileged locus for token-level prediction shaping
\citep{geva2021ffn,geva2022ffn,dai2022knowledge}.

\paragraph{Setup.}
We model hybrid thinking as conditional language modeling with an observed mode variable
\(r \in \{0,1\}\), where \(r=0\) denotes \nothink{} and \(r=1\) denotes \think{}.
Let
\begin{align*}
\alpha &= \text{shared parameters (embeddings, attention, norms, LM head)}, \\
\beta_0 &= \text{no-think expert parameters}, \\
\beta_1 &= \text{think expert parameters}.
\end{align*}
Thus the full parameter vector is
\begin{align*}
\theta = (\alpha,\beta_0,\beta_1).
\end{align*}

For an example \((x,y,r)\), the \ple{} model defines
\begin{align*}
p_\theta(y \mid x,r)
=
\prod_{t=1}^{|y|}
p_\theta\!\left(y_t \mid x,y_{<t}; \alpha,\beta_r\right).
\end{align*}

\paragraph{Observed-route formulation.}
Introduce an expert variable \(z \in \{0,1\}\).
Because routing is deterministic from the control token \(c\),
\begin{align*}
p(z \mid c) = \mathbf{1}\{z = r(c)\},
\end{align*}
where \(r(c)\) is the route selected by the control-token rule.
Therefore
\begin{align*}
p_\theta(y \mid x,c)
&=
\sum_{z \in \{0,1\}} p(z \mid c)\, p_\theta(y \mid x,z) \\
&=
\sum_{z \in \{0,1\}} \mathbf{1}\{z = r(c)\}\, p_\theta(y \mid x,z) \\
&=
p_\theta(y \mid x,r(c)).
\end{align*}
Hence the route variable is \emph{observed}, not latent.
In this observed-route setting, a learned router is not statistically necessary for identifying the routed conditional model; outside this setting (e.g., if the route is unobserved or must be inferred at test time), a learned router may still be useful or necessary.

\paragraph{Modeling assumption.}
The reduction from the control token \(c\) to the observed route \(r(c)\) assumes that, within this theoretical abstraction, \(c\) affects the conditional distribution only through route selection:
\begin{align*}
p_\theta(y \mid x,c,z) = p_\theta(y \mid x,z).
\end{align*}
Equivalently, once the route \(z\) is fixed, the control token carries no additional route-independent signal for prediction.
If, in the implemented model, \(c\) is also processed by the shared backbone through a distinct embedding or any other route-specific component, then the identities below should be read as an abstraction that isolates routing from those additional effects.

\paragraph{Dataset decomposition.}
Let
\begin{align*}
\mathcal{D} = \mathcal{D}_0 \sqcup \mathcal{D}_1,
\end{align*}
where \(\mathcal{D}_0\) contains no-think examples and \(\mathcal{D}_1\) contains think examples.
Define mode frequencies
\begin{align*}
\pi_r = \frac{|\mathcal{D}_r|}{|\mathcal{D}|}, \qquad \pi_0 + \pi_1 = 1.
\end{align*}
For each mode, define the expected autoregressive loss
\begin{align}
\mathcal{L}_r(\alpha,\beta_r)
=
\mathbb{E}_{(x,y)\sim \mathcal{D}_r}
\left[
-\sum_{t=1}^{|y|}
\log p_\theta(y_t \mid x,y_{<t}; \alpha,\beta_r)
\right].
\label{eq:mode_loss}
\end{align}
Then the full \ple{} objective is
\begin{align}
\mathcal{L}_{\mathrm{PLE}}(\alpha,\beta_0,\beta_1)
&=
\mathbb{E}_{(x,y,r)\sim \mathcal{D}}
\left[
-\sum_{t=1}^{|y|}
\log p_\theta(y_t \mid x,y_{<t}; \alpha,\beta_r)
\right] \\
&=
\sum_{r\in\{0,1\}} \pi_r \mathcal{L}_r(\alpha,\beta_r) \\
&=
\pi_0 \mathcal{L}_0(\alpha,\beta_0)
+
\pi_1 \mathcal{L}_1(\alpha,\beta_1).
\label{eq:ple_decomposition}
\end{align}

Unless stated otherwise, all exact statements below refer to the example-weighted objective in \eqref{eq:ple_decomposition}, with \(\pi_r = |\mathcal{D}_r|/|\mathcal{D}|\).
The later length-asymmetry discussion explicitly switches to a token-summed training heuristic.
If training instead uses token-weighted normalization, or adds auxiliary router losses, cross-expert regularizers, or other route-coupling terms, then the gradient and Hessian identities below acquire additional additive terms.

\begin{proposition}[Exact expert-gradient decoupling]
Under \eqref{eq:ple_decomposition},
\begin{align*}
\nabla_{\beta_0}\mathcal{L}_{\mathrm{PLE}}
&=
\pi_0 \nabla_{\beta_0}\mathcal{L}_0,
\\
\nabla_{\beta_1}\mathcal{L}_{\mathrm{PLE}}
&=
\pi_1 \nabla_{\beta_1}\mathcal{L}_1,
\\
\nabla_{\alpha}\mathcal{L}_{\mathrm{PLE}}
&=
\pi_0 \nabla_{\alpha}\mathcal{L}_0
+
\pi_1 \nabla_{\alpha}\mathcal{L}_1.
\end{align*}
In particular,
\begin{align*}
\nabla_{\beta_0}\mathcal{L}_1 = 0,
\qquad
\nabla_{\beta_1}\mathcal{L}_0 = 0.
\end{align*}
\end{proposition}

\begin{proof}
Differentiate \eqref{eq:ple_decomposition} with respect to \(\beta_0\):
\begin{align*}
\nabla_{\beta_0}\mathcal{L}_{\mathrm{PLE}}
&=
\nabla_{\beta_0}
\Big(
\pi_0 \mathcal{L}_0(\alpha,\beta_0)
+
\pi_1 \mathcal{L}_1(\alpha,\beta_1)
\Big) \\
&=
\pi_0 \nabla_{\beta_0}\mathcal{L}_0
+
\pi_1 \underbrace{\nabla_{\beta_0}\mathcal{L}_1}_{=\,0} \\
&=
\pi_0 \nabla_{\beta_0}\mathcal{L}_0.
\end{align*}
The zero term appears because \(\mathcal{L}_1\) depends on \(\beta_1\) but not on \(\beta_0\).

Similarly,
\begin{align*}
\nabla_{\beta_1}\mathcal{L}_{\mathrm{PLE}}
&=
\nabla_{\beta_1}
\Big(
\pi_0 \mathcal{L}_0(\alpha,\beta_0)
+
\pi_1 \mathcal{L}_1(\alpha,\beta_1)
\Big) \\
&=
\pi_0 \underbrace{\nabla_{\beta_1}\mathcal{L}_0}_{=\,0}
+
\pi_1 \nabla_{\beta_1}\mathcal{L}_1 \\
&=
\pi_1 \nabla_{\beta_1}\mathcal{L}_1.
\end{align*}
For the shared parameters \(\alpha\), both mode losses depend on \(\alpha\), so
\begin{align*}
\nabla_{\alpha}\mathcal{L}_{\mathrm{PLE}}
=
\pi_0 \nabla_{\alpha}\mathcal{L}_0
+
\pi_1 \nabla_{\alpha}\mathcal{L}_1.
\end{align*}
\end{proof}

\paragraph{Pointwise consequence.}
For a single example \((x_i,y_i,r_i)\), define
\begin{align*}
\ell_i(\alpha,\beta_0,\beta_1)
=
-\sum_{t=1}^{|y_i|}
\log p_\theta(y_{i,t} \mid x_i,y_{i,<t}; \alpha,\beta_{r_i}).
\end{align*}
Then for any \(k\in\{0,1\}\),
\begin{align}
\frac{\partial \ell_i}{\partial \beta_k}
=
\begin{cases}
-\displaystyle\sum_{t=1}^{|y_i|}
\frac{\partial}{\partial \beta_{r_i}}
\log p_\theta(y_{i,t}\mid x_i,y_{i,<t};\alpha,\beta_{r_i}),
& k=r_i,\\[10pt]
0, & k\neq r_i.
\end{cases}
\label{eq:pointwise_grad}
\end{align}
So the inactive expert receives \emph{exactly zero} gradient on that example.

All curvature-based statements below assume the relevant local second derivatives exist and that differentiation may be interchanged with expectation in the usual way.

\paragraph{Partial block-diagonalization of the curvature.}
Let
\begin{align*}
H_{ab}^{(r)} := \nabla^2_{a,b}\mathcal{L}_r,
\qquad
a,b \in \{\alpha,\beta_0,\beta_1\}.
\end{align*}
Because \(\mathcal{L}_0\) does not depend on \(\beta_1\) and \(\mathcal{L}_1\) does not depend on \(\beta_0\),
\begin{align*}
\nabla^2_{\beta_0\beta_1}\mathcal{L}_{\mathrm{PLE}} = 0,
\qquad
\nabla^2_{\beta_1\beta_0}\mathcal{L}_{\mathrm{PLE}} = 0.
\end{align*}
Hence the Hessian takes the form
\begin{align}
\nabla^2 \mathcal{L}_{\mathrm{PLE}}
=
\begin{bmatrix}
\pi_0 H^{(0)}_{\alpha\alpha} + \pi_1 H^{(1)}_{\alpha\alpha}
&
\pi_0 H^{(0)}_{\alpha\beta_0}
&
\pi_1 H^{(1)}_{\alpha\beta_1}
\\[3pt]
\pi_0 H^{(0)}_{\beta_0\alpha}
&
\pi_0 H^{(0)}_{\beta_0\beta_0}
&
0
\\[3pt]
\pi_1 H^{(1)}_{\beta_1\alpha}
&
0
&
\pi_1 H^{(1)}_{\beta_1\beta_1}
\end{bmatrix}.
\label{eq:block_hessian}
\end{align}
Equation~\eqref{eq:block_hessian} is the precise sense in which \ple{} \emph{partially block-diagonalizes}
the hybrid-thinking optimization problem:
the expert blocks are exactly decoupled from each other, while both remain coupled to the shared backbone.

This block structure is exact only for the direct expert--expert interactions in the stated objective.
It does not imply full disentanglement over training, because both modes remain coupled through the shared parameters \(\alpha\): an update from one mode can change \(\alpha\), which in turn changes future gradients for the other mode.

\paragraph{Dense hybrid training as a conflicting shared-objective problem.}
Now consider a dense baseline with the same shared backbone \(\alpha\) but a \emph{single} shared MLP parameter vector \(\beta\).
Its objective is
\begin{align*}
\mathcal{L}_{\mathrm{dense}}(\alpha,\beta)
=
\pi_0 \mathcal{L}_0(\alpha,\beta)
+
\pi_1 \mathcal{L}_1(\alpha,\beta).
\end{align*}
Let
\begin{align*}
g_0 := \nabla_\beta \mathcal{L}_0(\alpha,\beta),
\qquad
g_1 := \nabla_\beta \mathcal{L}_1(\alpha,\beta).
\end{align*}
Then the dense expert gradient is
\begin{align*}
\nabla_\beta \mathcal{L}_{\mathrm{dense}}
=
\pi_0 g_0 + \pi_1 g_1.
\end{align*}

\begin{proposition}[One-step interference criterion]
Consider one SGD step with learning rate \(\eta>0\):
\begin{align*}
\Delta \beta_{\mathrm{dense}}
=
-\eta(\pi_0 g_0 + \pi_1 g_1).
\end{align*}
Then the first-order change in mode-0 loss is
\begin{align}
\mathcal{L}_0(\beta+\Delta \beta_{\mathrm{dense}})
=
\mathcal{L}_0(\beta)
-
\eta\Big(
\pi_0 \|g_0\|^2 + \pi_1 g_0^\top g_1
\Big)
+
O(\eta^2).
\label{eq:mode0_taylor}
\end{align}
Therefore the dense update increases the mode-\(0\) loss at first order whenever
\begin{align}
g_0^\top g_1 < -\frac{\pi_0}{\pi_1}\|g_0\|^2.
\label{eq:interference_condition}
\end{align}
The same statement holds symmetrically for mode \(1\).
\end{proposition}

\begin{proof}
Take a first-order Taylor expansion of \(\mathcal{L}_0\) around \(\beta\):
\begin{align*}
\mathcal{L}_0(\beta+\Delta \beta_{\mathrm{dense}})
=
\mathcal{L}_0(\beta)
+
g_0^\top \Delta \beta_{\mathrm{dense}}
+
O(\|\Delta \beta_{\mathrm{dense}}\|^2).
\end{align*}
Substitute the dense update:
\begin{align*}
g_0^\top \Delta \beta_{\mathrm{dense}}
&=
g_0^\top \big[-\eta(\pi_0 g_0 + \pi_1 g_1)\big] \\
&=
-\eta \pi_0 g_0^\top g_0 - \eta \pi_1 g_0^\top g_1 \\
&=
-\eta \pi_0 \|g_0\|^2 - \eta \pi_1 g_0^\top g_1.
\end{align*}
Substituting back yields \eqref{eq:mode0_taylor}.
The dense update increases \(\mathcal{L}_0\) at first order iff
\begin{align*}
-\eta \Big(\pi_0 \|g_0\|^2 + \pi_1 g_0^\top g_1\Big) > 0,
\end{align*}
which is equivalent to \eqref{eq:interference_condition}.
\end{proof}

By contrast, under the same full-objective weighting, the expected \ple{} update for the no-think expert is
\begin{align*}
\Delta \beta_0 = -\eta \pi_0 g_0,
\end{align*}
so
\begin{align*}
\mathcal{L}_0(\beta_0+\Delta\beta_0)
=
\mathcal{L}_0(\beta_0)
-
\eta \pi_0 \|g_0\|^2
+
O(\eta^2),
\end{align*}
with no cross-term involving \(g_1\).
If one instead conditions on a mode-\(0\) minibatch, the same calculation gives \(\Delta\beta_0=-\eta g_0\).
Thus \ple{} removes \emph{direct} cross-mode interference exactly in the expert parameters, while indirect coupling through the shared backbone remains.

\paragraph{A local quadratic theory of the dense compromise.}
To isolate the mechanism, fix the shared parameters at some local value \(\bar{\alpha}\) and write
\begin{align*}
L_r(\beta) := \mathcal{L}_r(\bar{\alpha},\beta).
\end{align*}

\begin{proposition}[Exact fixed-backbone dominance of separate experts]
Fix any shared parameter value \(\bar{\alpha}\), and define \(L_r(\beta):=\mathcal{L}_r(\bar{\alpha},\beta)\).
Then
\begin{align*}
\min_{\beta_0,\beta_1}\ \pi_0 L_0(\beta_0)+\pi_1 L_1(\beta_1)
\;\le\;
\min_{\beta}\ \pi_0 L_0(\beta)+\pi_1 L_1(\beta).
\end{align*}
\end{proposition}

\begin{proof}
The dense model is the constrained subset of the separate-expert model obtained by imposing \(\beta_0=\beta_1=\beta\).
Minimizing over the larger feasible set \((\beta_0,\beta_1)\) therefore cannot yield a larger objective value than minimizing over the restricted set \(\beta_0=\beta_1\).
\end{proof}

Let \(\beta_r^\star\) denote a local minimizer of \(L_r\), and let \(H_r \succeq 0\) be the local Hessian at that point.
A second-order Taylor expansion gives
\begin{align}
L_r(\beta)
\approx
L_r(\beta_r^\star)
+
\frac{1}{2}
(\beta-\beta_r^\star)^\top H_r (\beta-\beta_r^\star).
\label{eq:quadratic_mode}
\end{align}

For the remainder of this subsection, it is convenient to promote \eqref{eq:quadratic_mode} to the local surrogate
\begin{align*}
\widetilde{L}_r(\beta)
:=
L_r(\beta_r^\star)
+
\frac{1}{2}
(\beta-\beta_r^\star)^\top H_r (\beta-\beta_r^\star),
\label{eq:quadratic_surrogate}
\end{align*}
and to interpret all equalities below in this subsection as statements about the surrogate \(\widetilde{L}_r\), not the exact nonquadratic loss \(L_r\).

The dense expert objective becomes
\begin{align}
L_{\mathrm{dense}}(\beta)
&=
\pi_0 L_0(\beta)+\pi_1 L_1(\beta) \\
&\approx
\pi_0 L_0(\beta_0^\star)+\pi_1 L_1(\beta_1^\star)
+
\frac{\pi_0}{2}(\beta-\beta_0^\star)^\top H_0(\beta-\beta_0^\star)
+
\frac{\pi_1}{2}(\beta-\beta_1^\star)^\top H_1(\beta-\beta_1^\star).
\label{eq:dense_quad_start}
\end{align}

Expand the first quadratic term:
\begin{align*}
(\beta-\beta_0^\star)^\top H_0(\beta-\beta_0^\star)
=
\beta^\top H_0 \beta
-
2\beta^\top H_0 \beta_0^\star
+
(\beta_0^\star)^\top H_0 \beta_0^\star.
\end{align*}
Expand the second:
\begin{align*}
(\beta-\beta_1^\star)^\top H_1(\beta-\beta_1^\star)
=
\beta^\top H_1 \beta
-
2\beta^\top H_1 \beta_1^\star
+
(\beta_1^\star)^\top H_1 \beta_1^\star.
\end{align*}
Substitute both expansions into \eqref{eq:dense_quad_start}:
\begin{align}
L_{\mathrm{dense}}(\beta)
&\approx
C
+
\frac{1}{2}
\beta^\top(\pi_0 H_0+\pi_1 H_1)\beta
-
\beta^\top(\pi_0 H_0 \beta_0^\star + \pi_1 H_1 \beta_1^\star),
\label{eq:dense_quad_expanded}
\end{align}
where
\begin{align*}
C
=
\pi_0 L_0(\beta_0^\star)+\pi_1 L_1(\beta_1^\star)
+
\frac{\pi_0}{2}(\beta_0^\star)^\top H_0 \beta_0^\star
+
\frac{\pi_1}{2}(\beta_1^\star)^\top H_1 \beta_1^\star.
\end{align*}

Differentiate \eqref{eq:dense_quad_expanded} with respect to \(\beta\):
\begin{align*}
\nabla_\beta L_{\mathrm{dense}}(\beta)
=
(\pi_0 H_0+\pi_1 H_1)\beta
-
(\pi_0 H_0 \beta_0^\star+\pi_1 H_1 \beta_1^\star).
\end{align*}
Setting the gradient to zero gives
\begin{align*}
(\pi_0 H_0+\pi_1 H_1)\beta_{\mathrm{dense}}^\star
=
\pi_0 H_0 \beta_0^\star+\pi_1 H_1 \beta_1^\star,
\end{align*}
hence
\begin{align}
\beta_{\mathrm{dense}}^\star
=
(\pi_0 H_0+\pi_1 H_1)^{-1}
(\pi_0 H_0 \beta_0^\star+\pi_1 H_1 \beta_1^\star),
\label{eq:dense_optimum}
\end{align}
assuming \(\pi_0 H_0+\pi_1 H_1\) is invertible.

Equation~\eqref{eq:dense_optimum} shows that the dense MLP is forced to sit at a
curvature-weighted compromise between the two mode-specific optima.

\begin{theorem}[Mode-conflict gap for the local quadratic surrogate]
Under the local quadratic model \eqref{eq:quadratic_mode}, the best \ple{} solution in the expert subspace is
\begin{align*}
(\beta_0,\beta_1) = (\beta_0^\star,\beta_1^\star),
\end{align*}
with value
\begin{align*}
L_{\mathrm{PLE}}^\star
=
\pi_0 L_0(\beta_0^\star)+\pi_1 L_1(\beta_1^\star).
\end{align*}
The dense model incurs excess loss
\begin{align}
\Delta_{\mathrm{conflict}}
:=
L_{\mathrm{dense}}(\beta_{\mathrm{dense}}^\star)
-
L_{\mathrm{PLE}}^\star
=
\frac{1}{2}
\sum_{r\in\{0,1\}}
\pi_r
(\beta_{\mathrm{dense}}^\star-\beta_r^\star)^\top
H_r
(\beta_{\mathrm{dense}}^\star-\beta_r^\star)
\ge 0.
\label{eq:mode_conflict_gap}
\end{align}
\end{theorem}

\begin{proof}
For \ple{}, the expert-subspace objective is
\begin{align*}
L_{\mathrm{PLE}}(\beta_0,\beta_1)
=
\pi_0 L_0(\beta_0)+\pi_1 L_1(\beta_1).
\end{align*}
Because \(\beta_0\) and \(\beta_1\) are separated, minimization reduces to two independent problems,
so the optimum is attained at \((\beta_0^\star,\beta_1^\star)\), with value
\begin{align*}
L_{\mathrm{PLE}}^\star
=
\pi_0 L_0(\beta_0^\star)+\pi_1 L_1(\beta_1^\star).
\end{align*}

For the dense model, evaluate each mode loss at \(\beta_{\mathrm{dense}}^\star\):
\begin{align*}
L_r(\beta_{\mathrm{dense}}^\star)
=
L_r(\beta_r^\star)
+
\frac{1}{2}
(\beta_{\mathrm{dense}}^\star-\beta_r^\star)^\top
H_r
(\beta_{\mathrm{dense}}^\star-\beta_r^\star).
\end{align*}
Multiply by \(\pi_r\) and sum over \(r\in\{0,1\}\):
\begin{align*}
L_{\mathrm{dense}}(\beta_{\mathrm{dense}}^\star)
&=
\sum_{r\in\{0,1\}}
\pi_r L_r(\beta_{\mathrm{dense}}^\star) \\
&=
\sum_{r\in\{0,1\}}
\pi_r L_r(\beta_r^\star)
+
\frac{1}{2}
\sum_{r\in\{0,1\}}
\pi_r
(\beta_{\mathrm{dense}}^\star-\beta_r^\star)^\top
H_r
(\beta_{\mathrm{dense}}^\star-\beta_r^\star).
\end{align*}
Subtract \(L_{\mathrm{PLE}}^\star\) to obtain \eqref{eq:mode_conflict_gap}.
Since each \(H_r \succeq 0\), every quadratic term is nonnegative, so
\(\Delta_{\mathrm{conflict}} \ge 0\).
\end{proof}

Equation~\eqref{eq:mode_conflict_gap} is the formal version of the core intuition for the local quadratic surrogate: if think and no-think prefer different MLP parameters in directions with nonzero local curvature, then a dense hybrid MLP pays a nonzero compromise penalty.
\ple{} removes that surrogate penalty in the expert subspace by allowing each mode to occupy its own local optimum.

\begin{corollary}[Equal-curvature closed form for the local quadratic surrogate]
If \(H_0 = H_1 = H \succ 0\), then
\begin{align*}
\beta_{\mathrm{dense}}^\star
=
\pi_0 \beta_0^\star + \pi_1 \beta_1^\star,
\end{align*}
and the compromise penalty simplifies to
\begin{align}
\Delta_{\mathrm{conflict}}
=
\frac{1}{2}
\pi_0\pi_1
(\beta_0^\star-\beta_1^\star)^\top
H
(\beta_0^\star-\beta_1^\star).
\label{eq:equal_curvature_gap}
\end{align}
\end{corollary}

\begin{proof}
Substitute \(H_0=H_1=H\) into \eqref{eq:dense_optimum}:
\begin{align*}
\beta_{\mathrm{dense}}^\star
&=
(\pi_0 H+\pi_1 H)^{-1}
(\pi_0 H\beta_0^\star+\pi_1 H\beta_1^\star) \\
&=
\big((\pi_0+\pi_1)H\big)^{-1}
H(\pi_0\beta_0^\star+\pi_1\beta_1^\star) \\
&=
\pi_0\beta_0^\star+\pi_1\beta_1^\star,
\end{align*}
since \(\pi_0+\pi_1=1\).

Now define \(\Delta\beta := \beta_0^\star-\beta_1^\star\).
Then
\begin{align*}
\beta_{\mathrm{dense}}^\star-\beta_0^\star
&=
(\pi_0\beta_0^\star+\pi_1\beta_1^\star)-\beta_0^\star
=
-\pi_1(\beta_0^\star-\beta_1^\star)
=
-\pi_1\Delta\beta,
\\
\beta_{\mathrm{dense}}^\star-\beta_1^\star
&=
(\pi_0\beta_0^\star+\pi_1\beta_1^\star)-\beta_1^\star
=
\pi_0(\beta_0^\star-\beta_1^\star)
=
\pi_0\Delta\beta.
\end{align*}
Plug these into \eqref{eq:mode_conflict_gap}:
\begin{align*}
\Delta_{\mathrm{conflict}}
&=
\frac{1}{2}
\Big[
\pi_0(-\pi_1\Delta\beta)^\top H (-\pi_1\Delta\beta)
+
\pi_1(\pi_0\Delta\beta)^\top H (\pi_0\Delta\beta)
\Big] \\
&=
\frac{1}{2}
\Big[
\pi_0\pi_1^2 \Delta\beta^\top H \Delta\beta
+
\pi_1\pi_0^2 \Delta\beta^\top H \Delta\beta
\Big] \\
&=
\frac{1}{2}
\pi_0\pi_1(\pi_0+\pi_1)\Delta\beta^\top H \Delta\beta \\
&=
\frac{1}{2}
\pi_0\pi_1\Delta\beta^\top H \Delta\beta.
\end{align*}
\end{proof}

Equation~\eqref{eq:equal_curvature_gap} is especially useful for intuition:
the dense compromise penalty grows with
(i) the frequency of both modes, through \(\pi_0\pi_1\),
(ii) the geometric separation between their preferred optima, and
(iii) the curvature in the directions where they disagree.

\paragraph{Why sequence-level routing matters.}
For sequence-level routing, each example has a single route \(r_i\), so
\begin{align*}
\ell_i
=
\sum_{t=1}^{|y_i|}
\ell_{i,t}(\alpha,\beta_{r_i}).
\end{align*}
Thus
\begin{align}
\nabla_{\beta_k}\ell_i
=
\mathbf{1}\{k=r_i\}
\sum_{t=1}^{|y_i|}
\nabla_{\beta_k}\ell_{i,t}.
\label{eq:sequence_level_grad}
\end{align}
A single response updates at most one expert.

If routing were token-level, with \(r_{i,t}\in\{0,1\}\), then
\begin{align*}
\ell_i^{\mathrm{token}}
=
\sum_{t=1}^{|y_i|}
\ell_{i,t}(\alpha,\beta_{r_{i,t}}),
\end{align*}
and hence
\begin{align*}
\nabla_{\beta_k}\ell_i^{\mathrm{token}}
=
\sum_{t=1}^{|y_i|}
\mathbf{1}\{k=r_{i,t}\}
\nabla_{\beta_k}\ell_{i,t}.
\end{align*}
Now the same response can update both experts, weakening the clean decoupling in
\eqref{eq:pointwise_grad} and \eqref{eq:sequence_level_grad}.
This is why sequence-level routing is the natural match to a turn-level \think{}/\nothink{} interface.

\paragraph{Why identical initialization matters.}
\ple{} initializes both experts from the same dense source MLP:
\begin{align*}
\beta_0^{(0)} = \beta_1^{(0)} = \beta_{\mathrm{src}}.
\end{align*}
If the two routes differ only through the duplicated expert weights at initialization, then before fine-tuning
\begin{align*}
p_\theta(y \mid x,r=0)
=
p_\theta(y \mid x,r=1)
=
p_{\mathrm{src}}(y \mid x).
\end{align*}
When newly introduced control-token embeddings or other route-specific components are present, this equality need not hold exactly.
In that case, the intended claim is narrower: the two experts start from identical weights, so specialization is not driven by random asymmetry in the expert parameters.

The update identities below are exact for plain SGD with matched learning rates and no optimizer state.
With momentum, Adam, gradient clipping, or asynchronous minibatch schedules, the same conclusion is qualitative rather than algebraically exact.

After one SGD step on a no-think batch and one SGD step on a think batch,
\begin{align*}
\beta_0^{(1)} &= \beta_{\mathrm{src}} - \eta g_0^{(0)}, \\
\beta_1^{(1)} &= \beta_{\mathrm{src}} - \eta g_1^{(0)}.
\end{align*}
Subtracting yields
\begin{align*}
\beta_1^{(1)} - \beta_0^{(1)}
=
-\eta\big(g_1^{(0)} - g_0^{(0)}\big).
\end{align*}
After \(T\) paired updates,
\begin{align*}
\beta_1^{(T)} - \beta_0^{(T)}
=
-\eta
\sum_{s=0}^{T-1}
\big(g_1^{(s)} - g_0^{(s)}\big).
\end{align*}
Thus expert divergence is driven purely by the cumulative difference between mode-specific gradients.

\paragraph{Why duplicating only the MLP can be sufficient.}
Let \(u_t\) denote the shared residual-stream representation entering a routed MLP block,
and let \(f_{\beta_r}(u_t)\) denote the routed expert output.
Abstract the remainder of the network (all later layers plus the LM head) as a differentiable map \(G\).
Then the route-\(r\) logits can be written locally as
\begin{align*}
o_t^{(r)}
=
G\big(u_t + f_{\beta_r}(u_t)\big).
\end{align*}
A first-order expansion around \(u_t\) gives
\begin{align}
o_t^{(r)}
\approx
G(u_t)
+
J_G(u_t)\, f_{\beta_r}(u_t),
\label{eq:mlp_linearization}
\end{align}
where \(J_G(u_t)\) is the Jacobian of \(G\) at \(u_t\).
This linearization should be interpreted as a local sensitivity argument rather than a global proof of sufficiency. It is accurate when the downstream map \(G\) is locally smooth around \(u_t\) and the routed residual update \(f_{\beta_r}(u_t)\) remains within a neighborhood where higher-order terms are small; in affine or shared linear-region regimes, the relation is exact.
Hence the think/no-think logit difference is approximately
\begin{align}
o_t^{(1)} - o_t^{(0)}
\approx
J_G(u_t)\Big(f_{\beta_1}(u_t)-f_{\beta_0}(u_t)\Big).
\label{eq:route_logit_difference}
\end{align}
Equations~\eqref{eq:mlp_linearization}--\eqref{eq:route_logit_difference}
make the architectural intuition explicit:
the shared backbone can keep constructing a common contextual representation \(u_t\),
while the mode-specific MLPs directly control how that representation is pushed toward different token distributions.
In this sense, \ple{} separates \emph{expression policy} more than \emph{core competence}.

\paragraph{A length-induced asymmetry heuristic.}
Let mode \(r\) have \(N_r\) examples with average target length \(m_r\),
and let \(g_{i,t}^{(r)}\) denote the per-token gradient contribution in the expert subspace.
A token-summed dense objective produces an epoch-level update
\begin{align*}
\Delta \beta_{\mathrm{dense}}
\propto
-
\sum_{r\in\{0,1\}}
\sum_{i=1}^{N_r}
\sum_{t=1}^{m_i^{(r)}}
g_{i,t}^{(r)}.
\end{align*}
If think targets are much longer than no-think targets, then even with similar example counts,
the dense expert receives more cumulative update mass from think-mode tokens.
\ple{} changes this to
\begin{align*}
\Delta \beta_0
\propto
-
\sum_{i=1}^{N_0}
\sum_{t=1}^{m_i^{(0)}}
g_{i,t}^{(0)},
\qquad
\Delta \beta_1
\propto
-
\sum_{i=1}^{N_1}
\sum_{t=1}^{m_i^{(1)}}
g_{i,t}^{(1)},
\end{align*}
so long think traces can no longer wash over the no-think expert.
This gives a conditional explanation for why \ple{} may especially strengthen the cleanliness and stability of \nothink{} mode when optimization effectively scales with token count and think targets are substantially longer.

\end{document}